\documentclass[accepted]{ecai}
\usepackage{graphicx}
\usepackage{latexsym}
\usepackage{tikz}
\usepackage{caption}
\usepackage[square,sort,comma,numbers]{natbib}

\usepackage[british]{babel}

\usepackage{defs}
\usepackage{xcolor}

\definecolor{d_green}{RGB}{0,200,0}
\definecolor{d_red}{RGB}{200,0,00}

\usepackage[colorinlistoftodos]{todonotes}

\begin{document}

\begin{frontmatter}

\title{Learning Task Automata for Reinforcement Learning \\ Using Hidden Markov Models\thanks{Accepted to the 26th European Conference on Artificial Intelligence (ECAI 2023). All authors are listed alphabetically.}}

\author[A]{\fnms{Alessandro}~\snm{Abate}}
\author[A,B]{\fnms{Yousif}~\snm{Almulla}}
\author[A]{\fnms{James}~\snm{Fox}}
\author[A]{\fnms{David}~\snm{Hyland}}
\author[A]{\fnms{Michael}~\snm{Wooldridge}}

\address[A]{University of Oxford}
\address[B]{Microsoft Azure Quantum}
\address[]{\{aabate, james.fox, david.hyland, mjw\}@cs.ox.ac.uk, yalmulla@microsoft.com}

\begin{abstract}
Training reinforcement learning (RL) agents using scalar reward signals is often infeasible when an environment has sparse and non-Markovian rewards. Moreover, handcrafting these reward functions before training is prone to misspecification. We learn non-Markovian finite task specifications as finite-state `task automata' from episodes of agent experience within environments with unknown dynamics. First, we learn a product MDP, a model composed of the specification's automaton and the environment's MDP (both initially unknown), by treating it as a partially observable MDP and employing a hidden Markov model learning algorithm. Second, we efficiently distil the task automaton (assumed to be a deterministic finite automaton) from the learnt product MDP. Our automaton enables a task to be decomposed into sub-tasks, so an RL agent can later synthesise an optimal policy more efficiently. It is also an interpretable encoding of high-level task features, so a human can verify that the agent's learnt tasks have no misspecifications. Finally, we also take steps towards ensuring that the automaton is environment-agnostic, making it well-suited for use in transfer learning. 
\end{abstract}

\end{frontmatter}

\section{Introduction}\label{sec:intro}

Reinforcement Learning (RL) can be prohibitively slow (sample inefficient) at learning an optimal policy when the reward signal is sparse and non-Markovian because of the credit assignment problem. 
However, this setting is common: any task where no reward is given until a sequence of sequential sub-tasks is completed.
Consider a medical procedure where the ordering of sub-task completion matters or where one can only enter a room if one already has the key.

Three existing approaches for improving learning in this setting are hierarchical RL \citep{pateria2021hierarchical}, which allows agents to plan at various levels of abstraction; transfer learning, which utilises knowledge learnt from similar tasks \citep{taylor2007cross}; and temporal logic planning approaches, which guide the agent's exploration by focusing it on the MDP fragment that satisfies a linear temporal logic (LTL) property \citep{hasanbeig2018logically, thiebaux2006decision, jothimurugan2021compositional}. The latter is similar to Icarte et al \cite{rmfulltoro2020}'s reward machines (finite-state machines that represent non-Markovian reward functions) because an LTL property is often represented as an automaton.

Task specifications are often separable, so an automaton (or reward machine) allows an optimal policy to be found more efficiently by breaking the task into Markovian sub-tasks. Recent work (see Section \ref{sec:related}) has recognised that this automaton is usually a priori unknown, so they learn it concurrently with finding an optimal policy. We focus on \emph{how to best learn the task automaton (TA) representing a task specification in sparse, non-Markovian reward environments.} To our knowledge, ours is the first model-based approach; our pipeline learns a model of the underlying MDP along with the TA. The aforementioned work on temporal logic planning, sample-efficient model-based RL approaches \citep{wang2019benchmarking}, or other methods \citep{baier2008principles} can then synthesise an optimal policy.

\paragraph{Contributions:} Our algorithmic pipeline learns both a task specification as a TA (encoded as a deterministic finite automaton (DFA)) and a model of the MDP, from episodes of agent experience within an environment with unknown dynamics and sparse, non-Markovian reward. This improves sample efficiency in three ways: a TA exposes a task specification's separability, so it allows for solving sub-tasks independently; our model-based approach helps the agent learn an optimal policy more efficiently; and we remove environmental bias from the learned TA for better transfer learning~\citep{taylor2007cross}.

We learn the TA via an intermediate product MDP structure, composed of the environment's `\emph{spatial}' MDP and the task specification's TA. The product MDP is partially observable -- the agent only observes its state in the spatial MDP and its reward, not the TA-state (its progress through the unknown task). 
The product MDP is learnt using the Baum-Welch algorithm \citep{baum1966statistical} by first learning an estimate of the spatial MDP using a uniform prior. Then, this learnt spatial MDP is used as an inductive bias for learning the full product MDP. In Section \ref{sec:exp}, we show that our approach is more efficient than existing SAT-based approaches to learning TAs \citep{deepsynth,xu2020joint,verginis2022joint,abadi2020learning,corazza2022reinforcement} by comparing with Biermann and Feldman's \cite{biermann1972synthesis} SAT-based algorithm.

Once the product MDP is learnt, distilling the TA via our `Cone Lumping' method is computationally cheap. This makes the product MDP useful for transfer learning. If the environment changes, the agent only needs to update the affected part of the product MDP before efficiently re-distilling the~TA. 

This paper also addresses the reward misspecification problem \citep{amodei2016concrete} by using a direct task specification to encode a complex, non-Markovian task rather than a scalar reward function. This is because scalar reward functions are brittle with respect to small environmental changes \citep{vazquez2018learning}, do not signal a task's separability, and can be difficult for a human to interpret and verify whether they result in the desired behaviour. In contrast, our learnt TA reveals the full task specification. Since our pipeline works in unknown environments where many high-level tasks may exist, agents using our pipeline can learn multiple tasks in one TA. Although we focus on deterministic TAs in this paper (i.e., DFAs), our approach can be straightforwardly extended to learning probabilistic automata \citep{rabin1963probabilistic, dohmen2022inferring}. 

Although the paper is self-contained, for brevity, we relegate many proposition proofs to our supplementary material, Appendix A.

\section{Setup}
An agent's interaction with its environment can be viewed at many levels of abstraction, amongst which a low and fine-grained description is often modelled by a \emph{Markov Decision Process (MDP)}. However, agents can also recognise higher-level features of the environment (e.g., chairs or tables). These are included in a \emph{labelled MDP} (in Section \ref{sec:MDPenv}) using a set of (atomic) propositional variables from an alphabet of labels $\mathcal{AP}$, which are assigned truth values at every MDP state via a labelling function~$L$. Task specifications refer to sequential compositions of these high-level features and can be encoded as TAs (in Section \ref{sec:taskspec}), which graphically represent the structure of tasks that an agent can get reward from completing. A TA and the MDP can be combined into a single structure, a product MDP (in Section \ref{sec:productMDP}) to provide a holistic representation of the agent's exploration.   

\subsection{MDP Environment}
\label{sec:MDPenv}

\begin{definition}
  \label{def:MDP}
  A \emph{Markov Decision Process (MDP)} is a tuple $\mathcal{M} = (\mathcal{S}, \mathcal{A}, s_0, P, R, \gamma)$, where $\mathcal{S}$ and $\mathcal{A}$ are finite sets of states and actions, $s_0$ is the initial state (or distribution over states), $P \colon \mathcal{S}\times \mathcal{A} \times \mathcal{S} \rightarrow [0,1]$ is the transition probability distribution, $R: (\mathcal{S} \times \mathcal{A})^+ \times \mathcal{S} \rightarrow \reals$ is a (non-Markovian) reward function\footnote{This must encompass the initial transition, as evidenced by the use of the non-empty repetition operator $+$.}, and $\gamma \in [0,1)$ is the discount factor. A \emph{labelled MDP} is a tuple $\mathcal{M} = (\mathcal{S}, \mathcal{A}, s_0, P, R, \gamma, \mathcal{AP}, L)$; an MDP enriched with $\mathcal{AP}$, a finite alphabet of atomic propositions (labels), and $L \colon \mathcal{S} \rightarrow 2^\mathcal{AP}$, a labelling function. 
\end{definition} 

A sequence of states and actions $s_0,a_0,s_1\ldots,a_{n-1},s_n$ is called a \emph{trajectory}. A \emph{state} or \emph{action trajectory} is the restriction to just the sequence of states  $(s_t)_{t=0}^n \in \mathcal{S}^+$ or actions $(a_t)_{t=0}^n \in \mathcal{A}^+$. 
A state trajectory is \emph{attainable} if ${\prod_{t=0}^n P(s_{t+1}\mid s_t, a_t) > 0}$ for some action trajectory $(a_t)_{t=0}^n$, and \emph{impossible} otherwise. A trajectory is accompanied by a \emph{reward sequence} $(r_t)_{t=0}^n$ and a sequence of labels, called a \emph{trace}, $(\ell_t)_{t=0}^n \in (2^\mathcal{AP})^*$, where $L(s_t) = \ell_t$ for all $t \leq n$. A \emph{(history-based) policy} ${\pi \colon (\mathcal{S} \times \mathcal{A})^+ \times \mathcal{A} \to [0,1]}$ is a function mapping trajectories to probability distributions over actions, where $\pi(a_t\mid s_0,a_0,\ldots, s_t)$ gives the probability of the agent choosing action $a_t$ at state $s_t$ given the (historical) trajectory $s_0,a_0,\ldots, s_t$. \emph{Memoryless policies} are a special case where the distribution over actions only depends on the last state (standard in RL because storing entire histories is intractable). The agent starts at $s_0$ in each exploration episode and, at each time step $t$, in state $s_t$, it selects action $a_t$ with probability $\pi(a_t \mid s_t)$, before transitioning to a new state $s_{t+1}$ with probability $P(s_{t+1} \mid s_t, a_t)$ whilst receiving a reward ${r_t = R(s_0, a_0,\ldots,s_t, a_t, s_{t+1})}$. The agent's task is to find a policy $\pi$ that maximises the expected sum of discounted reward $\mathbb{E}_\pi\left[\Sigma^n_{t=0}\gamma^t r_t\right]$.
A policy $\pi$ is \emph{fully mixed} if every action at every state is chosen with non-zero probability (i.e., $\pi(a \mid s)>0$ for all $s \in \mathcal{S}$ and $a \in \mathcal{A}$). We assume the agent observes state trajectories, traces, and reward sequences, as it explores the labelled MDP.

Given the initial state $s_0$, fixing the agent's policy $\pi$ in an MDP $\mathcal{M}$ induces a Markov chain (MC) $\mathcal{M}_\pi$ if for all state transitions, the reward function's output is independent of the last action (i.e., $R(\zeta, a_i, s_{t+1}) = R(\zeta, a_j, s_{t+1})$ for all $a_i, a_j \in A$, where $\zeta \in (\mathcal{S} \times \mathcal{A})^+$). Definition \ref{def:MDP}'s reward function is \emph{non-Markovian} because it depends upon the full trajectory up until that point. A \emph{Markovian reward function} $R: \mathcal{S} \times \mathcal{A} \times \mathcal{S} \rightarrow \reals$ only depends on the last transition. State-based rewards could also be handled straightforwardly.

\begin{example}\label{example}
  An RL agent is in Figure \ref{fig:casemdp}'s environment. The task is to collect coffee \faCoffee~for the guest on a couch \faCouch~before turning on the TV \faTv~and then ascending the stairs~$\textstairs$ (Figure \ref{fig:casedfa}). The carpet \faBorderNone~and book \faBook~do not affect whether the agent can achieve its goal. The arrow next to the silhouette indicates the initial state.
  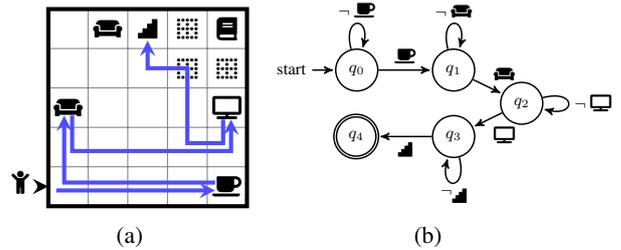
\begin{figure}[tb]
    \centering
   \begin{subfigure}[b]{.4\linewidth}
   \centering
   \vspace{-1em}
   \resizebox{1\linewidth}{!}{
   \begin{tikzpicture}[scale=0.54]
    \draw[step=1cm,gray] (0,0) grid (5, 5);
    
    \node at (4.5,0.5) {\faCoffee};

    
    \node at (3.5,4.5) {\faBorderNone};
    \node at (4.5,4.5) {\faBook};
    \node at (3.5,3.5) {\faBorderNone};
    \node at (4.5,3.5) {\faBorderNone};

    
    \node at (4.5,2.5) {\faTv};
    \node at (1.5,4.5) {\faCouch};
    \node at (0.5,2.5) {\faCouch};

    \node at (-0.7,.6) {\agent};
    
    \node at (2.5,4.5) {$\stairwayupfilled$};
    
    \draw[ultra thick] (0,0) rectangle (5, 5);
    
    \path[ultra thick,->, >=stealth] (-.28,.5) edge node [left] {} (0,.5);
    
    \draw[ultra thick, ->, >=stealth, draw=blue!70!white] (0.2,0.4) -- (4.2,0.4);
    \draw[ultra thick, ->, >=stealth, draw=blue!70!white] (4.2,0.6) -- (0.4,0.6) -- (0.4,2.3);
    \draw[ultra thick, ->, >=stealth, draw=blue!70!white] (0.6,2.3) -- (0.6,1.4) -- (4.6,1.4) -- (4.6,2.2);
    \draw[ultra thick, ->, >=stealth, draw=blue!70!white] (4.4,2.2) -- (4.4,1.6) -- (3.5, 1.6) -- (3.5,3.5) -- (2.5,3.5) -- (2.5,4.2);
\end{tikzpicture}
   }
   \subcaption{}
   \label{fig:casemdp}
   \end{subfigure}
   \begin{subfigure}[b]{.49\linewidth}
    \vspace{-1em}
   \resizebox{1.1\linewidth}{!}{
    \begin{tikzpicture}[->,>=stealth',shorten >=1pt,auto,node distance=1cm, thick] 

        \node[state,initial] (q_0)   {$q_0$}; 
        \node[state] (q_1) [right=of q_0] {$q_1$}; 
        \node[state] (q_2) [below right = 0.05cm and 0.7cm of q_1] {$q_2$}; 
        \node[state] (q_3) [below left =  0.05cm and 0.7cm of q_2] {$q_3$}; 
        \node[state,accepting] (q_4) [left =of q_3] {$q_4$}; 
        
        \path[->] 
        (q_0) edge [loop above] node {$\neg$ \faCoffee} ()   	
        (q_0) edge  node {\faCoffee} (q_1)
        
        (q_1) edge [loop above] node {$\neg$ \faCouch} ()   	
        (q_1) edge  node {\faCouch} (q_2)
        
        (q_2) edge [loop right] node {$\neg$ \faTv} ()   	
        (q_2) edge  node {\faTv} (q_3)
        
        (q_3) edge [loop below] node {$\neg \stairwayupfilled$} ()
        (q_3) edge node {$\stairwayupfilled$} (q_4);
        
    \end{tikzpicture}}
    \subcaption{}
   \label{fig:casedfa}
   \end{subfigure}
   \caption{ (a) The labelled MDP environment and (b) 5 state TA representing a task specification for Example 1.}
   \label{fig:casestudy}
 \end{figure}
 \end{example}
Example 1's MDP environment is shown in Figure \ref{fig:casemdp}. The agent's possible actions $\mathcal{A}$ are to move in the adjacent up, down, left, and right directions; the MDP's state space cardinality $\vert \mathcal{S} \vert$ is 25; and the transition dynamics are initially unknown (as discussed later, stochasticity in the dynamics derives from randomised action execution). The set of (atomic) propositional variables $\mathcal{AP}\equiv \{\text{\faCoffee,\faBorderNone,\faCouch,\faTv,\faBook,}\textstairs\}$ represent the high-level features of the environment. Empty cells are implicitly labelled $L(s) = \emptyset$ (i.e., the MDP state $s \in \mathcal{S}$ has no high-level environmental feature of interest). 

If the agent only gets a positive reward upon completing the full task specification, this cannot be expressed with a Markovian reward function with respect to the MDP's state space $\mathcal{S}$ because the agent cannot go between states labelled \faCoffee~and \faCouch~in one timestep, let alone complete all four subtasks. So, the agent must learn an optimal policy in an environment with a non-Markovian reward function that provides sparse feedback.
This challenge of planning using full \emph{trajectories} can be surmounted by recognising that the agent also observes \emph{traces} of high-level features in $\mathcal{AP}$. By maintaining a memory of complete subtasks (e.g., coffee collected) at this higher-level of abstraction (with smaller dimension), the agent can overcome the `curse of dimensionality' and synthesise an optimal policy more 
efficiently by exploiting the compositionality of many tasks.

\subsection{Task Specification}
\label{sec:taskspec}

\emph{Linear temporal logic (LTL)} or \emph{Linear Dynamic Logic (LDL)}, evaluated with respect to finite trajectories ($\text{LTL}_f$/$\text{LDL}_f$), is often used to express a task specification as temporally extended goals, or constraints on plans because this mitigates the sample inefficiency of RL in non-Markovian environments \citep{hasanbeig2018logically}.
In fact, Littman et al \cite{littman2017environment} argue that 
task specifications are superior to reward functions for specifying behaviour, as they are environment independent. We exploit the fact that the expressive power of both $\text{LDL}_f$ and $\text{LTL}_f$ is subsumed by the class of regular expressions, so they can be represented by deterministic or nondeterministic finite automata (DFAs/NFAs)~\citep{baier2008principles}.

\begin{definition}
  \label{def:DFA_NFA}
   A \emph{Non-Deterministic Finite Automaton (NFA)} is a tuple $\mathcal{N} = (\mathcal{Q}, q_0, \Sigma, \delta, \mathcal{F})$, where $\mathcal{Q}$ is a finite set of states, $q_0 \in \mathcal{Q}$ is the initial state, $\Sigma$ is a finite alphabet, $\delta: \mathcal{Q}\times \Sigma \rightarrow 2^\mathcal{Q}$ is a non-deterministic transition function, and $\mathcal{F}\subseteq \mathcal{Q}$ is the set of accepting states. If $|\delta(q, \alpha)|=1$ for all $q \in \mathcal{Q}$ and $\alpha \in \Sigma$, then $\mathcal{N}$ is deterministic and is called a \emph{Deterministic Finite Automaton (DFA)}.
\end{definition}

Let $\Sigma^*$ be the set of all finite words over $\Sigma$. A finite word $w = \alpha_1\alpha_2\ldots\alpha_n \in \Sigma^*$ is accepted by an NFA $\nfa$ if there exists a finite run $u_0,..., u_n \in \mathcal{Q}^*$ where $u_0 = q_0$, $u_{i+1} \in \delta(u_{i}, \alpha_{i+1})$ for $0 \leq i < n$, and $u_{n} \in \mathcal{F}$. 
In a DFA, only one path follows $w$, as each transition is deterministic, so the final state must be \emph{accepting}, otherwise the word is \emph{rejected}. The set of all accepted words is the language recognised by $\nfa$, denoted by $\mathcal{L}(\nfa)$. A DFA is trivially an NFA and any NFA can be converted into an equivalent DFA using Rabin and Scott \cite{rabin1959finite}'s subset construction method, so NFAs and DFAs both recognise the class of regular languages.

\begin{definition}
\sloppy A \emph{Task Automaton (TA)} $\auto$ is a DFA $(\mathcal{Q}, q_0, \Sigma, \delta, \mathcal{F})$ with $\Sigma = 2^\mathcal{AP}$, where $\mathcal{AP}$ is a set of propositional variables inherited from the labelled MDP $\mathcal{M}$.
\end{definition}

We define a TA as a DFA to make our presentation as clear as possible. However, other types of automaton may also be used, such as Icarte et al \cite{rmfulltoro2020}'s reward machines which, rather than having a set of accepting states (as in DFAs), outputs a reward function at each state $\delta_r : \mathcal{Q} \rightarrow [\mathcal{S} \times \mathcal{A} \times \mathcal{S} \rightarrow \reals]$. Our statistical approach to learning TAs also allows for probabilistic TAs \citep{rabin1963probabilistic}, where rewards exhibit both non-Markovian and stochastic dynamics.

The TA $\auto$ keeps a `memory' of traces $(\ell_t)_{t=0}^n \in (2^\mathcal{AP})^+$ that the exploring agent observes by transitioning according to the current truth assignment of the propositional variables. A trace is \emph{MDP-attainable} if there exists a state trajectory $(s_t)_{t=0}^n$ such that $L(s_t)=\ell_t$ for all $t$. The TA's accepting states $\mathcal{F}$ mark the satisfaction of some task specification. In other words, the agent completes a task via trace $(\ell_t)_{t=0}^n$ if, correspondingly, the TA starts 
in state $u_0 = q_0$, transitions according to $u_{t+1} = \delta(u_t, L(s_{t+1}))$ for $t \geq 0$, and $u_n \in \mathcal{F}$. 

The TA in Figure \ref{fig:casedfa} represents the task specification for Example~1 in terms of the high-level environmental features. Each move in the MDP is matched by a transition in the TA (recall that empty squares in Figure \ref{fig:casemdp} are labelled by $\emptyset$). For example, the (blue) marked trajectory in Figure \ref{fig:casemdp} represents an MDP-attainable accepting path in the TA and each `sub-path' denotes a TA `sub-task'. The sub-path $(\emptyset, \emptyset,\emptyset,\emptyset,\text{\faCoffee})$ denotes the path in the MDP that coresponds to traversing the $q_0 \rightarrow q_1$ edge in the TA. Note that the labels on self-loops in Figure \ref{fig:casedfa} (e.g., $\neg$\faCoffee) are syntactic sugar to denote all labels not associated with outgoing edges from that TA state. The TA is environment-agnostic because it applies to any labelled MDP that includes the set of all high-level features $\{\text{\faCoffee,\faCouch,\faTv,}\textstairs\}$ necessary for completing the corresponding task specification. Therefore, learning the TA can not only expedite policy synthesis within this environment, but it also helps transfer learning, where an agent wants to use what it knows about attaining reward in one environment to obtain reward in different environments -- e.g., if an agent learns it is necessary to pick up a key before opening a door in one house, then this is likely applicable to other houses.

\subsection{Product MDP}
\label{sec:productMDP}

\begin{definition}
  \label{def:prodmdp}
  Given an MDP without reward $\mathcal{M} \setminus R = (\mathcal{S}, \mathcal{A}, s_0, P)$ and a TA $\auto = (\mathcal{Q}, q_0, \Sigma, \delta, \mathcal{F})$, the \emph{product MDP} is the structure $\mathcal{M}\otimes \auto =  (\mathcal{S}^\otimes, \mathcal{A}, s_0^\otimes, P^\otimes, R^\otimes, \gamma)$, where $\mathcal{S}^\otimes = \mathcal{S}\times \mathcal{Q}$, $s_0^\otimes = \langle s_0,q_0\rangle$ is the initial state pair; $P^\otimes: \mathcal{S}^\otimes \times \mathcal{A} \times \mathcal{S}^\otimes \rightarrow [0,1]$ is the transition probability function, where $P^\otimes(\langle s_j,q_j\rangle\mid \langle s_i, q_i\rangle, a_i)=P(s_j \mid s_i, a_i)$ if $q_j=\delta(q_i,L(s_j))$ and 0 otherwise; and $R^\otimes: \mathcal{S}^\otimes \times \mathcal{A} \times \mathcal{S}^\otimes \rightarrow \{0,1\}$ is a reward function such that $R^\otimes(\langle s_t,q_t\rangle, a_t, \langle s_{t+1},q_{t+1}\rangle) = \chi_\mathcal{F}(q_{t+1})$, where $\chi_\mathcal{F}$ is the indicator function for the TA's set of accepting  states. Non-zero reward is obtained iff the agent transitions to an \emph{accepting state} of the product MDP (those related to an accepting TA state). Fixing any policy $\pi$ in the product MDP $\mathcal{M}\otimes \auto$, induces a \emph{product MC} $\mathcal{M}_\pi \otimes \auto$. 
\end{definition}

A product MDP $\mathcal{M}\otimes \auto$ more holistically represents the agent's exploration progress towards achieving a task because both the MDP $\mathcal{M}$ and the TA $\auto$ are undergoing transitions: in $\mathcal{M}$, these are the low-level dynamics over states in $\mathcal{S}$; in $\auto$, these represent the agent's progress towards completing a task specification. $\mathcal{M}\otimes \auto$'s reward function $R^\otimes$ is always Markovian because although a task may be non-Markovian with respect to $\mathcal{M}$, $\mathcal{M}\otimes \auto$ keeps track of what is happening at both $\mathcal{M}$'s and $\auto$'s level of abstraction. In other words, if $\mathcal{M}$ has a non-Markovian reward function $R$ representing a task encodable by $\auto$, then the non-Markovianity in $R$ can be resolved by synchronising $\mathcal{M}\otimes \auto$. The transitions in $\mathcal{M}\otimes \auto$ can be executed `under the hood' as the agent explores (i.e., no a priori knowledge about $\mathcal{M}\otimes \auto$ is required). Moreover, $\mathcal{M}\otimes \auto$ does not increase the agent's possible policy space because $\auto$ has deterministic transitions, so the trajectories (and therefore also policies) in $\mathcal{M}$ and $\mathcal{M}\otimes \auto$ are in a bijection. 

\section{Task Automaton Learning}
\label{sec:TAlearning}

We now give our three-step algorithmic pipeline for learning an MDP and a TA (i.e., the task specification) in unknown environments.

\begin{algorithm}[ht]
	\caption{Learning a TA in an unknown labelled MDP}
	\label{alg:alg}
	\begin{algorithmic}[1]
    \INPUT put agent into an (unknown) labelled MDP $\mathcal{M}$
    \OUTPUT TA $\auto^*$ that represents a task specification
    \State \textsc{ObsSeq} $\leftarrow$ collect episodes of corresponding trajectories, traces, and reward sequences.
    \Function{LearnProductMdp}{\textsc{ObsSeq}} \Comment \textbf{Step 1}
    \State \hskip\algorithmicindent \hspace*{-2.5em}Use an HMM/POMDP learning algorithm to learn estimates of \hskip\algorithmicindent \hspace*{0.5em}(i) the spatial MDP $\hat{\mathcal{M}}$'s transition probability distribution $\hat{P}$ \hskip\algorithmicindent \hspace*{0.5em}(ii) the product MDP~${\hat{\mathcal{M}}\otimes\hat{\auto}}$'s transition probability distribution $\hat{P}^\otimes$. 
    \EndFunction
    \Function{DistilTA}{$\hat{\mathcal{M}}\otimes\hat{\auto}$, \textsc{ObsSeq}} \Comment \textbf{Step 2}
        \State \hskip\algorithmicindent \hspace*{-2.5em}Determinise $\hat{\mathcal{M}}\otimes\hat{\auto}$ using Cone Lumping to return 
        the MDP-restricted TA $\hat{\auto}_{\hat{\mathcal{M}}}$
    \EndFunction
	\Function{PostProcess}{$\hat{\auto}_{\hat{\mathcal{M}}}$, \textsc{ObsSeq}} \Comment \textbf{Step 3}
	\State \hskip\algorithmicindent \hspace*{-2.5em}Remove environmental bias and minimise $\hat{\auto}_{\hat{\mathcal{M}}}$
	\EndFunction
	\end{algorithmic}
\end{algorithm}

\subsection{Step 1: Learn the Product MDP}
\label{sec:prodMDP}

Although the agent's interactions with its environment are fully modelled by the product MDP, the product MDP is only partially observable. The agent observes MDP states as well as the associated labels and rewards but does not observe the TA states it visits. So, Step 1 involves learning the product MDP's transition function $P^\otimes$ by viewing the agent's interaction with the (product MDP) environment as a partially-observable Markov Decision Process (POMDP). 
\begin{definition}
  \label{def:POMDP}
  A \emph{Partially-Observable Markov Decision Process (POMDP)} is a tuple $\mathcal{P} = (\mathcal{S}, \mathcal{A}, s_0, P, R, \gamma, \mathcal{O}, Z)$ where $(\mathcal{S}, \mathcal{A}, s_0, P, R, \gamma)$ is an MDP, $\mathcal{O}$ is an observation set, and $Z:  \mathcal{O}\times \mathcal{S}\rightarrow [0,1]$ is an observation probability function, where $Z(o \mid s)$ is the probability that observation $o\in\mathcal{O}$ is seen when the agent is at the  \emph{hidden} state $s\in \mathcal{S}$.  
\end{definition}

\begin{remark}
  \label{rem:POMDP}
  $\mathcal{P}^\otimes \coloneqq (\mathcal{S}^\otimes, \mathcal{A}, s_0^\otimes, P^\otimes, R^\otimes, \gamma, \mathcal{O}, Z)$ is a POMDP that models an agent interacting with a product MDP 
  $\mathcal{M}\otimes \auto = (\mathcal{S}^\otimes, \mathcal{A}, s_0^\otimes, P^\otimes, R^\otimes, \gamma)$, where the agent makes observations from the set $\mathcal{O} = \mathcal{S} \times \{0,1\}$ according to the deterministic observation function $Z : \mathcal{O}\times \mathcal{S}^\otimes \rightarrow \{0,1\}$. 
\end{remark}

Remark \ref{rem:POMDP} follows from Definitions \ref{def:prodmdp} and \ref{def:POMDP}. The agent's observations in the POMDP $\mathcal{P}^\otimes$ are from $\mathcal{S} \times \{0,1\}$, where the latter component is a binary value corresponding to the co-domain of the reward function $R^\otimes$. At each time step $t$, the agent transitions from a hidden state $s^\otimes_t = \langle s_t,q_t\rangle$ to $s^\otimes_{t+1} = \langle s_{t+1},q_{t+1}\rangle$ in the product MDP, but observes just $\big\langle s_{t+1}, \chi_\mathcal{F}(q_{t+1})\big\rangle$  with probability~1. The product MDP state $s_{t+1}^\otimes$ is \emph{partially obscured} from the agent. Recall $\chi_\mathcal{F}$ is the indicator function on the set of accepting states $\mathcal{F}$ in the TA $\auto$. If a policy $\pi$ is chosen by the agent interacting with the POMDP $\mathcal{P}^\otimes$, then the POMDP becomes a partially observable Markov chain $\mathcal{P}^\otimes_\pi$, known as a \emph{Hidden Markov Model (HMM)}~\citep{rabiner1989tutorial}. 

To learn the product MDP $\mathcal{M}\otimes \auto$, we can employ any algorithm for learning HMMs or POMDPs and apply it to $\mathcal{P}^\otimes$ or $\mathcal{P}^\otimes_\pi$. We use the Baum-Welch algorithm (BW) \citep{baum1966statistical}, due to its popularity and the ability to steer optimisation towards a desired local optimum (see Section \ref{sec:exp} and, for further details, Appendices C and~D). We learn $\mathcal{P}^\otimes_\pi$ in two stages. (i) Learn an estimate $\hat{P}$ of the true spatial MDP $\mathcal{M}$'s transition probability distribution $P$ using a fully mixed exploration policy $\pi$. This can be any fully mixed policy, but we chose a uniform random policy to assume no initial knowledge about the underlying model. Since we know $\pi$, we can always recover $P$ from $P_\pi$. (ii) Use this learnt estimate $\hat{P}$ as an inductive bias for learning the transition probability distribution $\hat{P}^\otimes$ of the full product MDP $\mathcal{M}\otimes \auto$. 

BW requires an initial guess on the number of hidden states ($\vert \mathcal{S}^\otimes \vert = \vert \mathcal{S} \vert \times \vert \mathcal{Q} \vert$). Since the agent can observe the MDP states, we only need guess the number of TA states $\hat{k}$. Furthermore, this guess only needs to be at least as high as the true $k$ (i.e., $\hat{k}\geq\mathcal{Q}$). If $\hat{k}$ is too high, it generates duplication in the learnt matrix, which is easily identified. For a sufficient upper bound for $\hat{k}$ one can: use the label set cardinality, use a TA simplicity assumption, or identify the minimum number of distinct labels needed for task completion using successful exploration episodes. We did the latter. Further details about the convergence criterion and other hyper-parameters (such as the episode length and the number of episodes) are in Section \ref{sec:exp}. 

To encode the learnt spatial MDP's transition probability distribution $\hat{P}$ as an inductive bias for learning $\hat{P}^\otimes$, we take the Kronecker product between a $\hat{k} \times \hat{k}$ identity matrix and $\hat{P}_\pi$ (the matrix induced from $\hat{P}$ and $\pi$) to form the basis of the initial estimate of the product MDP's transition matrix. For example, with a three-state TA:
\begin{equation*}
        I_{3\times 3} \otimes_K \hat{P}_\pi = 
        \begin{pmatrix}
        \hat{P}_\pi & \mathbf{0} & \mathbf{0}\\
        \mathbf{0} & \hat{P}_\pi & \mathbf{0}\\
        \mathbf{0} & \mathbf{0} & \hat{P}_\pi
        \end{pmatrix}.
\end{equation*}
Then, a small positive probability $\varepsilon > 0$ is added to each zero entry in the above construction to ensure that all expressions in the HMM learning procedure are well-defined. The probability $\varepsilon$ can take any positive value, but we chose a value inversely proportional to $\vert \mathcal{S}^\otimes \vert$ for consistency across experiments. Finally, each row is normalised so that the resulting initial estimate of the product MDP's transition matrix $\hat{P}^\otimes_\pi$ is a valid transition probability matrix.

Concentrating on learning $\hat{P}^\otimes$ (the process is the same for learning $\hat{P}$), BW begins with initial estimates for the hidden transition distribution $P^\otimes_\pi$ (an $\vert\mathcal{S}^\otimes\vert \times \vert\mathcal{S}^\otimes\vert$ matrix constructed from the inductive bias, as explained), the observation probability function $Z$ (an $\vert\mathcal{S}^\otimes\vert \times \vert\mathcal{O}\vert$ matrix), and the initial state distribution (a vector $\rho \in [0,1]^{\vert\mathcal{S}^\otimes\vert}$ defined as $\rho_j \equiv \Pr[s^\otimes_0 = s^\otimes_j]$). 
It then finds the $P^\otimes_\pi$ and $Z$ that maximises the likelihood of obtaining the set of observation sequences $O = (o_1,..., o_T) \in (\mathcal{O}^+)^T$:
\begin{equation*}
  \langle \hat{P}^\otimes_\pi, \hat{Z}\rangle = \argmax_{\langle P^\otimes,Z\rangle} \Pr\left(O \mid \langle P^\otimes_\pi,Z \rangle  \right).
\end{equation*}

\begin{definition}
  \label{def:obsequiv}
  Two product MDPs $\mathcal{M}_1\otimes \auto_1$ and $\mathcal{M}_2\otimes\auto_2$ with the same action spaces $\mathcal{A}_1=\mathcal{A}_2$ are \emph{observationally equivalent}, denoted $\mathcal{M}_1\otimes \auto_1 \cong \mathcal{M}_2\otimes\auto_2$, if for any observation sequence, any action trajectory, and all $\tau\in\N$, it holds that $
    {\Pr}^{\mathcal{P}^\otimes_1} \left[ (o_t)_{t=0}^\tau)|(a_t)_{t=1}^\tau)\right] = {\Pr}^{\mathcal{P}^\otimes_2} \left[ (o_t)_{t=0}^\tau)|(a_t)_{t=1}^\tau)\right],$
  where $\mathcal{P}_i^\otimes = (\mathcal{M}_i\otimes\auto_i, \mathcal{O}, Z)$. The observation spaces need not be identical for this to hold.
\end{definition}
Definition \ref{def:obsequiv} says that the learnt product MDP is, in general, not unique because several can have the same observational properties, i.e., no gathered experiment (episode) can distinguish between them. 
\begin{proposition}
  \label{prop:equiv}
  The observational equivalence relation $\cong$ is an equivalence relation since it is reflexive, symmetric, and transitive.
\end{proposition}

\begin{proposition}
\label{prop:lang}
Let $\mathcal{M}$ be a labelled MDP. For any two TAs $\auto,\auto'$, if $\mathcal{L}(\auto) = \mathcal{L}(\auto')$, then $\mathcal{M}\otimes\auto \cong \mathcal{M}\otimes\auto'$.
\end{proposition}

Proposition \ref{prop:equiv} says we can find the quotient set of all product MDPs with the same action spaces under the equivalence relation $\cong$. The equivalence class for the product MDP $\mathcal{M}\otimes\auto$ is denoted by $[\mathcal{M}\otimes\auto]$. Proposition \ref{prop:lang} enables us to formalise Step 1 of Algorithm 1.

\begin{mdframed}[backgroundcolor=gray!20, nobreak=true, innerleftmargin=8pt, innerrightmargin=8pt]
\tbf{Step 1:} Given state trajectories and reward sequences from interacting with the initially unknown product MDP $\mathcal{M}\otimes \auto$, learn an estimate $\hat{\mathcal{M}}\otimes\hat{\auto} \in [\mathcal{M}\otimes\auto]$.
\end{mdframed}

\begin{definition}
  The digraph $G=(E,V)$ underlying an MDP $\mathcal{M}$ with state-space $\mathcal{S}$ and transition probability function $P$ has $V = \mathcal{S}$ and $E = \{s_i\rightarrow s_j : \exists a_k \in \mathcal{A} \textrm{ s.t. } P(s_j\mid s_i, a_k) >0 \}$. $\mathcal{M}$ is \emph{structurally correct} with respect to another MDP $\mathcal{M}'$ if the digraph underlying $\mathcal{M}$ is isomorphic to $\mathcal{M}'$.
\end{definition}

Any algorithm which learns a maximum likelihood estimate (MLE) for the HMM parameters is asymptotically structurally correct, provided it converges to the true MLE and belongs in $[\mathcal{M}\otimes \auto]$. So, every hidden state of the learned model must correspond to a state $s^\otimes \in \mathcal{S}^\otimes$ in the true product MDP $\mathcal{M}\otimes \auto$ \cite{yang2015statistical}.

\subsection{Step 2: Distil TA from Product MDP}
\label{sec:distill}

The TA component $q$ of each product MDP state $s^\otimes = \langle s,q\rangle$ is hidden, but even when visible, the agent can only recover the TA fragment accessible using paths in the MDP.

\begin{definition}
\label{def:restrictedDFA}
  Given an MDP $\mathcal{M}$ and a TA $\auto$, the \emph{MDP-restricted TA} $\auto_\mathcal{M}$ is the sub-graph of $\auto$ covered by all MDP-attainable traces. We refer to paths within/outside of the MDP-restricted TA sub-graph as the \emph{MDP-attainable}/\emph{MDP-non-attainable} paths of the TA.
\end{definition}

\begin{proposition}
\label{prop:MDPTA}
 Suppose an agent is given a structurally correct product MDP $\hat{\mathcal{M}}\otimes\hat{\auto}$ belonging to $[\mathcal{M} \otimes \auto]$, where $\mathcal{M}$ is the true MDP underlying the product MDP, and the agent can observe $s$ and $q$ for all $\langle s,q\rangle \in \hat{\mathcal{S}}^\otimes$. Then, the agent can learn the MDP-restricted TA $\hat{\auto}_{\hat{\mathcal{M}}}$, but may not be able to learn the full TA~$\hat{\auto}$.
\end{proposition}

From Step 1, the agent knows the digraph of $\mathcal{M}\otimes\auto$ and an estimate $\hat{P}$ for the transition probability function of the true MDP $\mathcal{M}$. Therefore, we first compute a product MC $\hat{\mathcal{M}_\pi}\otimes\hat{\auto}$ by assigning to the agent a fully mixed policy~$\pi$. We now recognise that when the edges of the digraph underlying $\hat{\mathcal{M}_\pi}\otimes\hat{\auto}$ are labelled according to each edge's target state (recall that the agent observes traces), it results in an NFA whose accepting states are those in which reward 1 is obtained (see Appendix B for a worked example).

\begin{proposition}
  \label{prop:yieldNFA}
    Given a labelled product MC $\mathcal{M}_\pi\otimes\auto$ with binary reward, labelling every edge $\langle s_i, q\rangle\rightarrow\langle s_j, q_j\rangle$ according to the label of its target state's MDP component (i.e., in this case, $L(s_j) = \ell_j$) results in an NFA. 
\end{proposition}

This NFA can be converted into a DFA with language equal to the MDP-restricted TA $\auto_\mathcal{M}$ using the subset construction method \cite{rabin1959finite}.

\begin{proposition}
\label{prop:learntNFA}
  Let $\mathcal{N}$ be the NFA underlying the learnt product MC $\mathcal{M}_\pi\otimes\auto$. Then, $\mathcal{L}(\mathcal{N}) = \mathcal{L}(\auto_\mathcal{M})$, where $\auto_\mathcal{M}$ is the MDP-restricted true TA.
\end{proposition}

However, the subset-construction algorithm has worst-case exponential time complexity in the size of the NFA (i.e.,  $\mathcal{O}(2^{|\mathcal{S}^\otimes|})$). We therefore introduce our more efficient \textit{Cone Lumping} method for determinising NFAs that underlie product MDPs, which explicitly leverages their product structure. `Lumping' aligns with related literature (e.g., \citep{derisavi2003optimal}).
Our key insight is that out-going edges from a product state $\langle s,q\rangle \in \mathcal{S}^\otimes$ of a product MC $\mathcal{M}_\pi\otimes\auto$ that have the same label \emph{must} transition to the same next TA state. 
Every time this occurs, the product states are merged. When representing state transitions sequentially from left to right, this merging criterion looks for a \emph{cone} structure to merge all states on the right-hand side (Figure \ref{fig:cone}). 

\begin{lemma}[\tbf{Cone Lumping}]
  \label{lemma:cone}
    Given a labelled product MC $\mathcal{M}_\pi\otimes\auto$, let $E_l^{\langle s, q\rangle} = \{\langle s, q\rangle\xrightarrow[]{L(s_i)}\langle s_i, q_j\rangle : L(s_i) = \ell \}$ be the set of all out-going edges from state $\langle s, q\rangle$ in the digraph underlying $\mathcal{M}_\pi\otimes\auto$ with label $\ell$. Since any TA transition $q\xrightarrow[]{l} q_j$ is deterministic, then $q_j=q_k$ for any pair $q_j,q_k \in \{q': \langle s, q\rangle\xrightarrow[]{l}\langle s_i, q'\rangle \}$. Thus, for a product state $\langle s,q\rangle$, $|\{q' : \langle s, q\rangle\xrightarrow[]{L(s_i)}\langle s_i, q'\rangle \wedge L(s_i)=l \}|=1.$
\end{lemma}

\begin{figure}[h!]
  \centering
  \resizebox{0.3\linewidth}{!}{
  \begin{tikzpicture}[->,>=stealth',shorten >=1pt,auto,node distance=1cm, thick] 

	\node[state,minimum size=1.2cm] (source)   {$\langle s,q\rangle$}; 
	\node[state,minimum size=1cm] (A) [above right=of source, xshift=1cm,yshift=-0.4cm] {$\langle s_{1},q'\rangle$};
	\node (B) [right=of source, xshift=.8cm,yshift=0.25cm] {$\textbf{\vdots}$}; 
	\node (B1) [right=of source, xshift=-0.8cm,yshift=0.17cm] {$\textbf{\vdots}$}; 
	\node[state,minimum size=1cm,yshift=0.3cm] (C) [below =of A] {$\langle s_{n},q'\rangle$};

	\path[->] 
	(source) edge node[xshift=0.05cm,yshift=-0.05cm]{$\ell$}  (A)
	(source) edge node[below, xshift=-0.1cm]{$\ell$}  (C);
	
\end{tikzpicture}}
  \caption{Schematics of ``Cone Lumping''.} 
  \label{fig:cone}
\end{figure}
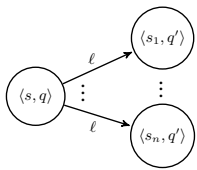

Performing Cone Lumping for all $\langle s,q\rangle\in \mathcal{S}^\otimes$ until a fixed point is reached (i.e., no more states can be merged), obtains a TA accepting (at minimum) $\mathcal{L}(\auto_\mathcal{M})$, and at most $\mathcal{L}(\auto)$ (by construction of the product MDP). This is because all non-determinism in the NFA underlying $\mathcal{M}\otimes\auto$ comes from the possible existence of multiple out-going edges with the same label from a product state. 

\begin{proposition}
\label{prop:cone}
Cone Lumping determinises any product MDP in $\mathcal{O}(|\mathcal{S}^\otimes|^3)$, exponentially faster than Rabin and Scott \cite{rabin1959finite}'s subset-construction $\mathcal{O}(2^{|\mathcal{S}^\otimes|})$ algorithm.
\end{proposition}

Cone Lumping therefore efficiently converts an NFA underlying a labelled MC into a DFA (our TA in this work, Section \ref{sec:exp}) and it assists in transfer learning because the hardness is confined to learning the product MDP. If the MDP or TA changes slightly, the agent only has to relearn the affected part of the product MDP, before using Cone Lumping to efficiently distil the new TA. 
In practice, this uses the transition probability distribution $\hat{P}^\otimes_\pi$ of the old MDP and TA's product MDP as an initial estimate for the new product MDP. The parts of the MDP and TA that remain the same will not need updating, so Baum-Welch only needs to update the parts that have changed. This speeds up convergence significantly \citep{yang2015statistical}. Finally, this process can also be useful outside of our work, e.g., for (bi)simulation or model reduction studies \citep{larsen1991bisimulation}.

\begin{mdframed}[backgroundcolor=gray!20, nobreak=true, innerleftmargin=8pt, innerrightmargin=8pt]
\tbf{Step 2:} Applying Cone Lumping to a labelled product MC $\mathcal{M}_\pi\otimes\auto$ for all $\langle s,q\rangle\in \mathcal{S}^\otimes$ until a fixed point is reached returns a TA $\auto_\mathcal{M}'$, where $\mathcal{L}(\auto_\mathcal{M}') = \mathcal{L}(\auto_\mathcal{M})$.
\end{mdframed}

\subsection{Step 3: Remove environmental bias and minimise the learnt TA}
\label{sec:envbias}

The TA distilled from the product MDP in Step 2 is in general not unique (cf. Section 3.1). If the TA is a DFA, as has been this paper's focus, it can be minimised using, for instance, Hopcroft \cite{hopcroft1971n}'s algorithm in $\mathcal{O}(n\log n)$. However, the TA might also not be minimal due to environmental bias. 
If a certain label \textit{must} be traversed by the agent to reach a necessary label for achieving the underlying task specification, automaton-learning algorithms cannot differentiate the importance of the two labels and will therefore assume that \textit{both} are necessary TA transitions, even when the former is unnecessary.

\begin{theo}\label{thm:nonuniqueTA}
    Given an arbitrary product MDP $\mathcal{M}\otimes \auto$, the TA distilled from $\mathcal{M}\otimes \auto$ may not be unique because there can exist $\auto'\neq \auto$, with or without $\mathcal{L}(\auto)\neq \mathcal{L}(\auto)$, such that $\mathcal{M}\otimes\auto \cong \mathcal{M}\otimes\auto'$. 
\end{theo}
\begin{proof}[Proof sketch by (counter)example] 
  Let $\mathcal{M}$ be the labelled MDP defined by the environment in Figure \ref{fig:casemdp} and $\auto$ and $\auto'$ be the TAs in Figures \ref{fig:TA1} and \ref{fig:TA2} respectively, where $\auto$ is the true TA. Then, $\auto'$ satisfies $\mathcal{M}\otimes\auto'\cong \mathcal{M}\otimes\auto$. In Figure \ref{fig:casestudy}, a \faBorderNone-label must be traversed immediately prior to reaching the \faBook-label, no matter which MDP path to \faBook~is taken. Because reward is only obtained upon reaching \faBook, there is no way to identify from the reward signal if the sequence (\faBorderNone \,,\, \faBook) is necessary for reward, or just reaching \faBook suffices. The agent does not know which of the two task specifications, expressed by the TAs $\auto$ and $\auto'$, is sufficient. 

\end{proof}

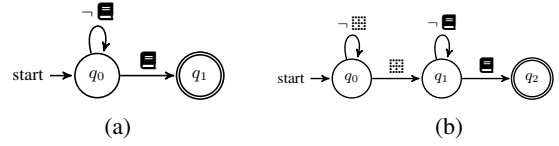
\begin{figure}[tb]
  \centering
 \begin{subfigure}[b]{.44\linewidth}
 \centering
\vspace{-2em}
 \resizebox{0.8\linewidth}{!}{
	\begin{tikzpicture}[->,>=stealth',shorten >=1pt,auto,node distance=1cm, thick] 

		\node[state,initial] (q_0)   {$q_0$}; 
		\node[state,accepting] (q_1) [right=of q_0] {$q_1$};

		\path[->] 
		(q_0) edge [loop above] node {$\neg$ \faBook} ()   	
		(q_0) edge  node {\faBook} (q_1);
		
	\end{tikzpicture}
 }
 \subcaption{}
 \label{fig:TA1}
 \end{subfigure}
 \begin{subfigure}[b]{.55\linewidth}
  \vspace{-2em}
 \resizebox{0.8\linewidth}{!}{
	\begin{tikzpicture}[->,>=stealth',shorten >=1pt,auto,node distance=1cm, thick] 

		\node[state,initial] (q_0)   {$q_0$}; 
		\node[state] (q_1) [right=of q_0] {$q_1$}; 
		\node[state,accepting] (q_2) [right=of q_1] {$q_2$}; 
		
		\path[->] 
		(q_0) edge [loop above] node {$\neg$ \faBorderNone} ()   	
		(q_0) edge  node {\faBorderNone} (q_1)
		
		(q_1) edge [loop above] node {$\neg$ \faBook} ()
		
		(q_1) edge node {\faBook} (q_2);
		
	\end{tikzpicture}}
  \subcaption{}
 \label{fig:TA2}
 \end{subfigure}
 \caption{Two TAs for Example 1's environment in Figure \ref{fig:casemdp} to show that the learned TA can be environmentally biased.}
 \label{fig:envbiasedTA}
\end{figure}

\begin{corollary}
  In general, it is impossible to verify that a TA $\auto'$ is the TA that underpins the product MDP $\mathcal{M}\otimes\auto$ (i.e., that $\auto'=\auto$), even given knowledge of $\mathcal{M}$ and $\mathcal{M}\otimes\auto$. 
\end{corollary}

Thus, to \emph{minimise} environmental bias, we must use counterfactual reasoning in our \emph{post-processing} step to select among the set of TAs $\{\auto': \mathcal{M}\otimes\auto'\cong \mathcal{M}\otimes\auto\}$ the TA that requires the smallest alphabet to explain its behaviour. 
\begin{mdframed}[backgroundcolor=gray!20, nobreak=true, innerleftmargin=8pt, innerrightmargin=8pt]
\tbf{Step 3:} For every label in $L(\mathcal{S})$, guess that it has no effect in the supposed true TA $\auto$ (i.e., the agent self-loops on every TA state). To check this guess, first merge any states of the learned-automaton $\auto'$ that have a non-loop edge between them with this label, then check if the resulting TA outputs the correct acceptance/rejection on every word in the set of observed sequences.
\end{mdframed}

This post-processing only removes labels that are never meaningful to $\auto$; it does not remove labels that \textit{are} necessary somewhere, but may still exist `wrongfully' in $\auto'$ due to environmental bias. Further post-processing can address this by looking for neighbourhoods of (possibly same-labelled) states that entirely surround another label, then checking, as above, whether these are necessary in $\auto$.

\section{Experimental Results}
\label{sec:exp}

Although the main contributions of this paper are theoretical, our experiments (available at \url{https://github.com/dkhyland/TALearner}) evaluate the efficacy of Algorithm 1's pipeline. To our knowledge, we are the first to focus on the task of learning both the TA and spatial MDP starting from no a-priori knowledge. Our experiments assess our pipeline's ability to scale up to 5x5 MDP grid-worlds and 5-state TAs and benchmark, where possible, to related work. These sizes are comparable to or beyond those handled in all related work (see Section \ref{sec:related}). 

Figure \ref{fig:exp_gridworld} illustrates the three grid worlds and the three-, four-, and five-state TAs used for our experiments. The grid worlds all share a common set of high-level environmental features $\mathcal{AP}\equiv \{\text{\faCoffee,\faBorderNone,\faCouch,\faTv,}\textstairs\}$, which have been distributed randomly across the grid. The agent can visit every action and state within the grid world.

\begin{figure}[h]
    \centering
    \begin{subfigure}[b]{.33\linewidth}
    \centering
    \resizebox{0.7\linewidth}{!}{\begin{tikzpicture}[scale=0.54]
    \draw[step=1cm,gray] (0,0) grid (3, 3);
    
    \node at (2.5,0.5) {\faCoffee};

    
    \node at (2.5,2.5) {\faBorderNone};

    
    \node at (2.5,1.5) {\faTv};
    \node at (0.5,1.5) {\faCouch};

    \node at (-0.7,.6) {\agent};
    
    \node at (1.5,2.5) {$\stairwayupfilled$};
    
    \draw[ultra thick] (0,0) rectangle (3, 3);
    
    \path[ultra thick,->, >=stealth] (-.28,.5) edge node [left] {} (0,.5);
    
\end{tikzpicture}}
    \subcaption{3x3 gridworld}
    \end{subfigure}%
    \begin{subfigure}[b]{.33\linewidth}
    \centering
    \resizebox{0.85\linewidth}{!}{\begin{tikzpicture}[scale=0.54]
    \draw[step=1cm,gray] (0,0) grid (4, 4);
    
    \node at (3.5,0.5) {\faCoffee};

    
    \node at (3.5,3.5) {\faBorderNone};

    
    \node at (3.5,1.5) {\faTv};
    \node at (0.5,1.5) {\faCouch};

    \node at (-0.7,.6) {\agent};
    
    \node at (1.5,3.5) {$\stairwayupfilled$};
    
    \draw[ultra thick] (0,0) rectangle (4, 4);
    
    \path[ultra thick,->, >=stealth] (-.28,.5) edge node [left] {} (0,.5);
    
\end{tikzpicture}}
    \subcaption{4x4 gridworld}
    \end{subfigure}%
    \begin{subfigure}[b]{.33\linewidth}
    \centering
    \resizebox{1\linewidth}{!}{\begin{tikzpicture}[scale=0.54]
    \draw[step=1cm,gray] (0,0) grid (5, 5);
    
    \node at (4.5,0.5) {\faCoffee};

    
    \node at (3.5,4.5) {\faBorderNone};
    \node at (4.5,4.5) {\faBorderNone};
    \node at (3.5,3.5) {\faBorderNone};
    \node at (4.5,3.5) {\faBorderNone};

    
    \node at (4.5,2.5) {\faTv};
    \node at (1.5,4.5) {\faCouch};
    \node at (0.5,2.5) {\faCouch};

    \node at (-0.7,.6) {\agent};
    
    \node at (2.5,4.5) {$\stairwayupfilled$};
    
    \draw[ultra thick] (0,0) rectangle (5, 5);
    
    \path[ultra thick,->, >=stealth] (-.28,.5) edge node [left] {} (0,.5);
    
\end{tikzpicture}}
    \subcaption{5x5 gridworld}
    \end{subfigure}%
    \newline
     \begin{subfigure}[b]{.99\linewidth}
    \centering
    
  \resizebox{0.43\linewidth}{!}{
\begin{tikzpicture}[->,>=stealth',shorten >=1pt,auto,node distance=1cm, thick] 

	\node[state,initial] (q_0)   {$q_0$}; 
	\node[state] (q_1) [right=of q_0] {$q_1$}; 
	\node[state,accepting] (q_2) [right=of q_1] {$q_2$}; 
	
	\path[->] 
	(q_0) edge [loop above] node {$\neg$ \faCoffee} ()   	
	(q_0) edge  node {\faCoffee} (q_1)
	
	(q_1) edge [loop above] node {$\neg \stairwayupfilled$} ()
	
	(q_1) edge node {$\stairwayupfilled$} (q_2);
	
\end{tikzpicture}}
    \resizebox{.56\linewidth}{!}{
\begin{tikzpicture}[->,>=stealth',shorten >=1pt,auto,node distance=1cm, thick] 

	\node[state,initial] (q_0)   {$q_0$}; 
	\node[state] (q_1) [right=of q_0] {$q_1$}; 
	\node[state] (q_2) [right=of q_1] {$q_2$}; 
	\node[state,accepting] (q_3) [right=of q_2] {$q_3$}; 
	
	\path[->] 
	(q_0) edge [loop above] node {$\neg$ \faCoffee} ()   	
	(q_0) edge  node {\faCoffee} (q_1)
	
	(q_1) edge [loop above] node {$\neg$ \faCouch} ()   	
	(q_1) edge  node {\faCouch} (q_2)
	
	(q_2) edge [loop above] node {$\neg \stairwayupfilled$} ()
	(q_2) edge node {$\stairwayupfilled$} (q_3);
	
\end{tikzpicture}}
    
    \vspace{0.5em}
    \resizebox{.5\linewidth}{!}{
    \begin{tikzpicture}[->,>=stealth',shorten >=1pt,auto,node distance=1cm, thick] 

        \node[state,initial] (q_0)   {$q_0$}; 
        \node[state] (q_1) [right=of q_0] {$q_1$}; 
        \node[state] (q_2) [below right = 0.05cm and 0.7cm of q_1] {$q_2$}; 
        \node[state] (q_3) [below left =  0.05cm and 0.7cm of q_2] {$q_3$}; 
        \node[state,accepting] (q_4) [left =of q_3] {$q_4$}; 
        
        \path[->] 
        (q_0) edge [loop above] node {$\neg$ \faCoffee} ()   	
        (q_0) edge  node {\faCoffee} (q_1)
        
        (q_1) edge [loop above] node {$\neg$ \faCouch} ()   	
        (q_1) edge  node {\faCouch} (q_2)
        
        (q_2) edge [loop right] node {$\neg$ \faTv} ()   	
        (q_2) edge  node {\faTv} (q_3)
        
        (q_3) edge [loop below] node {$\neg \stairwayupfilled$} ()
        (q_3) edge node {$\stairwayupfilled$} (q_4);
        
    \end{tikzpicture}}
    \subcaption{3, 4, and 5 state TAs}
    \end{subfigure}%
    \caption{(a), (b), and (c) show the grid worlds used in all experiments. The agent starts in the bottom left-hand corner of the grid and follows a random exploration policy over all episodes. (d) gives the TAs used for their respective task specifications: get coffee, then go upstairs; get coffee, serve to guest on a couch, then go upstairs; and get coffee, serve to guest on a couch, turn the TV on, then go upstairs.}
    \label{fig:exp_gridworld}
\end{figure}
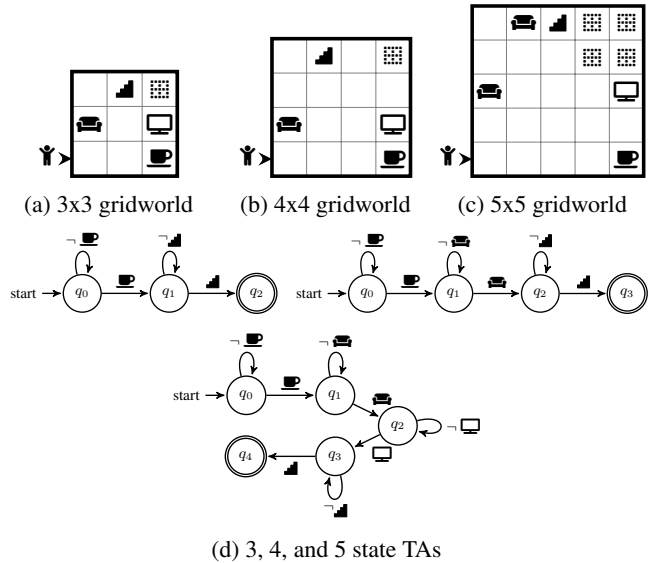

For all experiments, the convergence criterion was taken to be the point when the difference between the transition matrix estimates over one pass through the training data achieves a maximum absolute row sum less than $10^{-6}$. In practice, this is sufficient to ensure that a structurally correct product MDP is learned. 

The episode length and number of episodes used in the training data are hyperparameters that can affect both the convergence time and the structural correctness of the learnt transition matrix. In practice, choosing the episode length so that the agent achieves (on average) the goal in between $20-50\%$ of the episodes, ensures that the agent sufficiently explores all states in the product MDP. Additionally, a general trend was observed that for larger product MDPs, more training episodes were required to obtain a structurally correct estimate of the transition matrix. The number of episodes that were sufficient for convergence in our experiments was less in our model-based approach than in model-free approaches in related work (Section \ref{sec:related}). Exact hyperparameters used and full results are presented in Table \ref{tab:exp1_params} in Appendix E.

Step 1 of our algorithm (learning the product MDP) is implemented in C++ (see Appendices D and E for more details), and the remaining steps are implemented in Python. 
All experiments are the average of 3 runs (with negligible standard errors) and were conducted on an Intel Xeon Gold 6248 (2.50GHz) CPU. Cone Lumping (Step 2) determinises the NFA underlying the product MDP for the 5x5 grid and 5-state TA in Figure \ref{fig:casestudy} in 0.43s, and removing environmental bias (Step 3) takes 0.8ms. This also confirms the usefulness of these steps for other purposes (argued in Sections \ref{sec:distill} and~\ref{sec:envbias}).

We, therefore, turn to Step 1, Algorithm~1's bottleneck. Although, note that the question of whether there is an automaton with $n$ states that agrees with a finite set of data is NP-complete \citep{gold1978complexity}. Recall from Section \ref{sec:prodMDP} that we learn our product MDP estimate $\hat{\mathcal{M}}\otimes\hat{\auto}$ using the Baum-Welch algorithm in two stages. First, we learn an estimate of the spatial MDP $\mathcal{M}$'s transition probability distribution $\hat{P}$. Then, this is used as an inductive bias for learning the transition probability distribution $\hat{P}^\otimes$ of the full product MDP. Figure \ref{fig:experiment_1} demonstrates the merits of this \emph{two-stage} approach using two baselines: i) a naive \emph{uniform} approach, which learns a product MDP estimate using a uniform initialisation (each row in the transition matrix estimate was initialised to a uniform distribution), ii) an \emph{idealised} approach which assumes that the agent has full a priori knowledge of the true spatial MDP $\mathcal{M}$; this is encoded in the initial estimate of the product MDP's transition function. Run time was chosen to measure performance instead of the number of iterations because the latter depends on the number of training steps per episode and the number of episodes generated, which varies across experiments.

\begin{figure}[th]
    \centering
   \begin{subfigure}[b]{.47\linewidth}
   \centering
   \resizebox{1\linewidth}{!}{
   \begin{tikzpicture}
\begin{axis} [ybar = .05cm,
    bar width = 12pt,
    ymode=log,
    ylabel style={at={(axis description cs:0.08,.5)},anchor=south},
    xlabel=Grid size,
    ylabel=Time to convergence (s),
    xmin=0, xmax=40,
    ymin=1, ymax=60000,
    ytick pos=left,
    xtick pos=bottom,
    xtick={10,20,30},
    xticklabels={3x3,4x4,5x5},
    legend style={at={(0.25,0.73)},anchor=south,legend cell align=left}
]

\addplot[color=blue,fill=blue,opacity=0.6] coordinates {(10,106) (20,1840) (30,6279)};
\addlegendentry{Uniform}
\addplot[color=d_red,fill=d_red,opacity=0.6] coordinates {(10,26) (20,660) (30,2626)};
\addlegendentry{Two-stage}
\addplot[color=d_green,fill=d_green,opacity=0.6] coordinates {(10,39) (20,619)( 30,2351)};
\addlegendentry{Idealised}

\end{axis}
\end{tikzpicture}
   }
   \subcaption{}
   \label{fig:experiment_1}
   \end{subfigure}
\begin{subfigure}[b]{.52\linewidth}
   \centering
   \resizebox{1\linewidth}{!}{
   \usetikzlibrary{patterns}

\begin{tikzpicture}

\begin{axis} [ybar = .05cm,
    bar width = 9pt,
    ymode=log,
    ylabel style={at={(axis description cs:0.08,.5)},anchor=south},
    xlabel=Number of TA states,
    ylabel=Time to convergence (s),
    xmin=0, xmax=40,
    ymin=1, ymax=350000,
    ytick pos=left,
    xtick pos=bottom,
    xtick={10,20,30},
    xticklabels={3,4,5},
    legend style={at={(0.25,0.65)},anchor=south,legend cell align=left}
]

\addplot[color=blue,fill=blue,opacity=0.6] coordinates {(10,26) (20,382) (30,5233)};
\addlegendentry{3x3 (Ours)}
\addplot[color=orange,fill=orange,opacity=0.6] coordinates {(10,64) (20,18188)};
\addlegendentry{3x3 (BF)}

\addplot[color=d_red,fill=d_red,opacity=0.6] coordinates {(10,661) (20,3825) (30,38025)};
\addlegendentry{4x4}
\addplot[color=d_green,fill=d_green,opacity=0.6] coordinates {(10,2625) (20,16820) (30,68978)};
\addlegendentry{5x5}
\addplot[color=orange, fill=orange, opacity=0.6, postaction={pattern=north east lines}] coordinates {(23.6,150000)};

\node[] at (axis cs: 28,230000) {Timeout};

\end{axis}
 
\end{tikzpicture}
   }
   \subcaption{}
   \label{fig:experiment_2}
   \end{subfigure}   
   \caption{Step 1's convergence times (a) comparing our pipeline (red) with two baseline approaches, a uniform prior (blue) and an idealised prior in which the spatial MDP is known (green) (b) for varying MDP and TA sizes.}
   \label{fig:new}
 \end{figure}
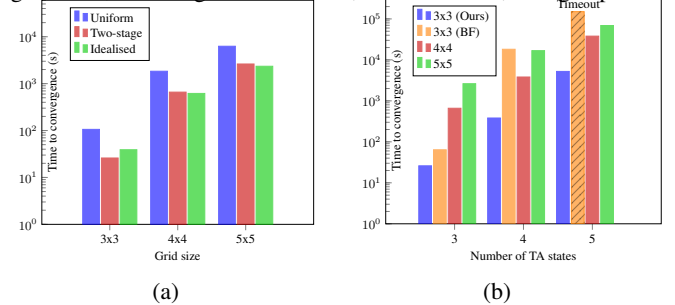

Figure \ref{fig:experiment_1} demonstrates our approach's superiority over the uniform baseline. We plot the convergence times of Step 1 for our approach (in red) against the two baselines for the 3 state TA and 3 different grid sizes. As the size of the grid increases, the time taken to learn an accurate estimate of the product MDP trends towards the idealised baseline. This is because the time taken to learn $\hat{P}$ in the first stage becomes negligible compared with learning the TA. Thus, learning the TA via a larger (and therefore ostensibly more complex) structure first, the product MDP, does not noticeably increase the time complexity of the process. This also demonstrates why our process is useful for transfer learning (Section \ref{sec:distill}). 

Figure \ref{fig:experiment_2} shows the dependence of the convergence time on the sizes of the MDP environment and TA. Increasing the TA size increases the number of sub-tasks the agent must complete before receiving reward. This Figure also shows how our algorithm's efficiency compares with related work. We presented the codebases of five related works (see Appendix F for details) the same sequences used in our smallest experiment (3x3 grid with a 3-state TA), but none of these could successfully learn the TA (or threw a recursion error). We were, however, able to perform some benchmarks against an implementation of Biermann and Feldman \cite{biermann1972synthesis}'s SAT-based algorithm in the libalf library to infer a hypothesis reward machine consistent with the data for the 3- and 4-state TAs. We used the same set of sequences that were used in our algorithm for comparison, which were labelled as positive or negative sequences depending on whether the agent received a reward. With a 4-state TA, ours had a lower convergence time in a 4x4 gridworld than BF managed in even a 3x3 gridworld, and BF timed out at $150,000$s when trying to learn the 5-state TA in the 3x3 gridworld.

\section{Related Research}
\label{sec:related}

SAT-based approaches \citep{deepsynth,xu2020joint,verginis2022joint,abadi2020learning,corazza2022reinforcement} generate a hypothesis automaton (equivalent to a temporal specification) and then verify whether this agrees with the agent's observations during exploration. We showed that these are less efficient than our approach by benchmarking against the only available codebase that learnt an automaton from our traces. Meanwhile, Icarte et al \cite{lrm-toro2019} use \textit{Tabu search}, which relies on already knowing a partial model, and Furelos-Blanco et al \cite{furelos2020induction} use an \textit{Answer Set Programming} algorithm, which assumes a known upper bound for the maximum finite distance between TA states. Our assumptions are weaker than these as we are only required to guess an upper bound on the number of TA states (in Section \ref{sec:exp}, we explained why this is easy to do) and we do not require any a priori knowledge about the spatial MDP. Also, all of these approaches are local, whereas ours is complete in the sense that our TA captures all sequential behaviour seen in the traces (high-level features).

\cite{vazquez2018learning,saqur2022reward,lauffer2022learning} all consider probabilistic approaches to learning specifications. Their maximum entropy approaches (inspired by Ziebart et al \cite{ziebart2008maximum}'s variant of inverse reinforcement learning) start from an initial pool of candidate specifications. In contrast, our maximum-likelihood approach does not a priori require any structure of the specification or the spatial MDP environment. 
Meanwhile, \cite{gaonNMR, rens2020online, dohmen2022inferring, topper2022active, xu2021active} use Angluin \cite{angluin1987learning}'s $L^*$ algorithm to learn a TA, relying on an \textit{oracle} for equivalence and membership queries. We assume that the agent cannot access an oracle and must learn the TA fully autonomously, which aligns with the standard setup of model-free RL (note that $L^*$ was not originally developed for RL applications). Because this means our training data cannot be provided as input to $L^*$, we cannot benchmark fairly against these approaches.

Learning an automaton from data is NP-hard \citep{gold1978complexity}, so under standard complexity assumptions, the time taken to solve this problem will necessarily grow exponentially as a function of the size of the TA.
The most appropriate method will depend on the use-case as different approaches choose to relax different assumptions. For example, Corazza et al \cite{corazza2022reinforcement} and Dohmen et al \cite{dohmen2022inferring} extend the SAT-based approach to noisy rewards and the $L^*$ approach to learning probabilistic automata, respectively.

Our method is most appropriate in the following situations. First, when one does not assume the presence of an oracle (required for $L^*$). Second, if the reward environment is sparse, i.e, where the agent must follow a sequence of non-rewarding steps before getting a reward if and only it completes the full task. Third, when sample-efficiency is critical (as in many real-world scenarios) since our model-based approach requires fewer episodes to learn the TA. Fourth, for transfer learning: our approach can readily handle the underlying MDP or TA changing slightly. For example, suppose the initial task specification is to get coffee for the guest on the couch. Then, the furniture is rearranged and the task changes from collecting coffee to collecting tea. Only the affected part of the product MDP needs to be re-learnt before performing the subsequent computationally cheap steps (2 and 3) in our algorithmic pipeline. Other methods would need to start learning the new TA from scratch each time.
Finally, the stochastic nature of our approach has advantages over the discrete approaches of related work; e.g, we can easily incorporate active-learning to more intelligently guide the agent's exploration \citep{anderson2005active}. Furthermore, our stochastic approach allows for entropy-based uncertainty quantification and it will always generate an estimate of the solution even with few exploration episodes. 

\section{Conclusions and Future work}
\label{sec:conclusions}
We have introduced an unsupervised pipeline for an RL agent to learn task specifications, encoded as deterministic, finite automata known as ‘task automata’, in unknown environments with sparse and non-Markovian rewards. The agent first learns a product MDP -- a composition of the task automaton with the environment's MDP before distilling out a task automaton using our efficient `Cone Lumping' method. Finally, our post-processing routine, which is useful for any method that learns automata in RL, minimises environmental bias and simplifies the learnt task automaton.

As future work, we are extending to environments with multiple tasks (i.e., cases where an agent can achieve reward by satisfying multiple different task specifications); introducing active learning to further improve the efficiency of the rate-limiting Step 1; and trying different techniques for learning HMMs/POMDPs (e.g., spectral learning \citep{hsu2012spectral}). We are also interested in multi-agent settings, where agents get rewarded according to different tasks or have to optimally share information to learn a common task. Finally, we intend to extend to cases where the task automata encode $\omega$-regular LTL properties, which require evaluation over infinite trajectories.

\ack Fox was supported by the EPSRC Centre for Doctoral Training in Autonomous Intelligent Machines and Systems (Reference: EP/S024050/1) and Wooldridge was supported
by a UKRI Turing AI World Leading Researcher Fellowship (Reference: EP/W002949/1). 

\bibliography{refs}

\appendix
\onecolumn
\pagebreak

\begin{center}
\textbf{\huge Learning Task Automata for Reinforcement Learning Using Hidden Markov Models (Supplementary Material)}
\end{center}

\appendix
\section{Proofs of Technical Results}
\setcounter{proposition}{1}
\setcounter{lemma}{0}
\setcounter{theorem}{0}
\setcounter{corollary}{0}

\begin{proposition}
Let $\mathcal{M}$ be a labelled MDP. For any two TAs $\auto,\auto'$, if $\mathcal{L}(\auto) = \mathcal{L}(\auto')$, then $\mathcal{M}\otimes\auto \cong \mathcal{M}\otimes\auto'$.
\end{proposition}
\begin{proof}
  Consider the two product Markov processes $\mathcal{M}\otimes\auto$ and $\mathcal{M}\otimes\auto'$ constructed by composing the MDP $\mathcal{M}$ with each TA $\auto$ and $\auto'$. At each time step $t$, the agent's observation $o_t$ is a pair in $\mathcal{S} \times \{0,1\}$. Therefore, observation sequences restricted to the first component of the observation pair are state trajectories, and restricted to the second component are reward sequences $\left(\chi_\mathcal{F}(q_{t})\right)^\tau_{t=0} \in \{0,1\}^*$. The probabilities of state trajectories must be equal for $\mathcal{M}\otimes\auto$ and $\mathcal{M}\otimes\auto'$ since they are constructed from identical environments $\mathcal{M}$, and the probabilities of reward sequences must be equal because the languages of $\auto$ and $\auto$ are identical by assumption.
\end{proof}

\begin{proposition}
  Suppose an agent is given a structurally correct product MDP $\hat{\mathcal{M}}\otimes\hat{\auto}$ belonging to $[\mathcal{M} \otimes \auto]$, where $\mathcal{M}$ is the true MDP underlying the product MDP, and the agent can observe $s$ and $q$ for all $\langle s,q\rangle \in \hat{\mathcal{S}}^\otimes$. Then, the agent can learn the MDP-restricted TA $\hat{\auto}_{\hat{\mathcal{M}}}$, but may not be able to learn the full TA~$\hat{\auto}$.
\end{proposition}
\begin{proof}
  \sloppy Start with the product MDP $\hat{\mathcal{M}}\otimes\hat{\auto}$, where the agent can observe the MDP and TA state components ($s$ and $q$) of each product MDP state $s^\otimes = \langle s,q\rangle$. Induce a product MC $\hat{\mathcal{M}_\pi}\otimes\hat{\auto}$ by employing a random exploration policy $\pi$. Note that any exploration policy can be used here, provided that every action is chosen at every state with non-zero probability.
  For each transition $P^\otimes(\langle s_j,q_j\rangle\mid \langle s_i, q_i\rangle) > 0$ in $\hat{\mathcal{M}_\pi}\otimes\hat{\auto}$, add the edge $q_i\stackrel{L(s_j)}{\longrightarrow}q_j$ to the candidate TA (unless it already exists). After going through all such transitions in the product MDP, all TA transitions accessible by an MDP-attainable trace must have been exhausted. We have therefore created the MDP-restricted TA $\hat{\auto}_{\hat{\mathcal{M}}}$. $\hat{\auto}_{\hat{\mathcal{M}}}$ may not be equal to the true underlying TA $\auto$ because there might exist a path in $\auto$ that corresponds with a trace that cannot be generated by an agent exploring in $\mathcal{M}$.
\end{proof}

\begin{proposition}
    Given a labelled product MC $\mathcal{M}_\pi\otimes\auto$ with binary reward, labelling every edge $\langle s_i, q\rangle\rightarrow\langle s_j, q_j\rangle$ according to the label of its target state's MDP component (i.e., in this case, $L(s_j) = \ell_j$), results in an NFA. 
\end{proposition}
\begin{proof}
  An NFA is a tuple $(\mathcal{Q}, q_0, \Sigma, \delta, \mathcal{F})$ (Definition 2). The set of states $\mathcal{Q}$ is the set of states of the product MDP. The state $q_0$ is the same as the initial state of the product MDP (since the initial states of the MDP $\mathcal{M}$ and TA $\auto$ are known). $\Sigma = 2^\mathcal{AP}$ since these are the possible labels of the edges. $\delta$ is indeed a non-deterministic transition function since from any state $s\in \mathcal{S}$, there may be multiple $s_i\in \mathcal{S}$ such that a transition from $s$ to $s_i$ is possible, and $L(s_i)=\ell$. That is, multiple possible next states may exist, which all have the same label. Finally, the accepting states $\mathcal{F}$ are the accepting states of the product MC $\mathcal{M}_\pi\otimes\auto$ -- i.e., the states where the reward observed is non-zero.
\end{proof}

\begin{proposition}
  Let $\mathcal{N}$ be the NFA underlying the learnt product MC $\mathcal{M}_\pi\otimes\auto$. Then, $\mathcal{L}(\mathcal{N}) = \mathcal{L}(\auto_\mathcal{M})$, where $\auto_\mathcal{M}$ is the MDP-restricted true TA.
\end{proposition}
\begin{proof}
  $\mathcal{N}$ accepts precisely those words that terminate at an accept state of the product MC $\mathcal{M}_\pi\otimes\auto$, and because a product state is an accept state if and only if its TA-component is accepting, the language recognised by $\mathcal{N}$ must be contained in the language of the true underlying TA. Moreover, because the set of paths in $\mathcal{N}$ underlying $\mathcal{M}_\pi\otimes\auto$ is precisely the set of MDP-attainable paths of the true TA, the language of $\mathcal{N}$ is precisely the language of the MDP-restricted TA $\auto_\mathcal{M}$.
\end{proof}

\begin{proposition}
Cone Lumping determinises any product MDP in $\mathcal{O}(|\mathcal{S}^\otimes|^3)$, exponentially faster than Rabin and Scott \cite{rabin1959finite}'s subset-construction $\mathcal{O}(2^{|\mathcal{S}^\otimes|})$ algorithm.
\end{proposition}
\begin{proof}
  Our implementation of Cone Lumping (Lemma 1) uses iterations that involve looping over all states in $\mathcal{S}^\otimes$ and all out-going transitions at each state and then merging any successor-states\footnote{A state $\langle s',q'\rangle$ such that $\langle s,q\rangle\rightarrow \langle s',q'\rangle$ is called a \textit{successor-state} of $\langle s,q\rangle$.} with the same label from $2^\mathcal{AP}$ into one equivalence class that is used as the new state. The iterations end when no states can be `lumped' into the same equivalence class (corresponding to the same TA state). In the worst case, only one state is subtracted from the NFA upon each iteration, implying that there are at most $|\mathcal{S}^\otimes|-1$ iterations. This also implies that, in the worst case, the number of successor-states that need to be checked on iteration $i$ is $(|\mathcal{S}^\otimes|-i)^2$: there are $|\mathcal{S}^\otimes|-i$ states left in the model, and from any of these there are at most $|\mathcal{S}^\otimes|-i$ out-going edges (in the worst case, the underlying graph is fully-connected). Thus, the worst-case number of check-operations before this algorithm terminates is $\sum_{i=2}^{|\mathcal{S}^\otimes|}i^2 = \frac{1}{6} |\mathcal{S}^\otimes|(|\mathcal{S}|+1)(2|\mathcal{S}|+2) - 1 = O(|\mathcal{S}^\otimes|^3)$. In the common case where there are precisely $k$ permitted out-going transitions from any state in the model (e.g., gridworld environments), this complexity reduces to $\sum_{i=2}^{|\mathcal{S}^\otimes|} i\cdot k = \frac{k}{2}|\mathcal{S}^\otimes|(|\mathcal{S}^\otimes|+1) - k=O(|\mathcal{S}^\otimes|^2)$.
\end{proof}

\section{Cone Lumping Example}

A product MDP is of size $\vert \mathcal{S} \times \mathcal{Q} \vert$, so we use a simple $3\times3$ grid-world with a three state TA (Figure \ref{fig:test1mdp}) to use as an easy demonstration of cone lumping. The task specification here again can still not be expressed by a Markovian reward function because the agent cannot go from either \faCoffee~to the stairs $\textstairs$ in one step. A product MDP in this environment can be learnt using Step 1 of our pipeline. We then induce a product MC by assigning to the agent a fully mixed policy $\pi$. Figure \ref{fig:test1prodmdp} then shows the NFA underlying this product MC, where the edges have been labelled as instructed in Proposition 5 and $\langle s_{21}, q_2\rangle$ is the only accepting state.

For the Cone Lumping step, any two outgoing edges from a product state of this NFA underlying a labelled product MC that have the same label must transition to the same next TA state (cf. Section 3.2). Observe that this is the case in Figure \ref{fig:test1prodmdp}; for example, every edge labelled with $\emptyset$ outgoing from $\langle s_{00},q_0\rangle \in \mathcal{S}^\otimes$ ends up in a TA state with $q_0$ as its TA component. Cone Lumping efficiently converts the NFA in Figure \ref{fig:test1prodmdp} into the DFA in Figure \ref{fig:test1mdp}.

In the main text, we also mentioned that Cone Lumping helps transfer learning because the hardness has been confined to learning the product MDP. If the spatial MDP or the TA changes slightly, the agent only has to relearn the affected part of the product MDP, then can use Cone Lumping to efficiently find the new TA. For instance, if after learning the TA from exploring the spatial MDP in Figure \ref{fig:test1mdp}, the agent was then put in a new $3\times3$ grid-world with no carpets present. The agent only has to re-learn the parts of the product MDP associated with MDP states $s_{20}$ and $s_{22}$ before applying our efficient Cone Lumping method again. It would then find that the old TA still applies.

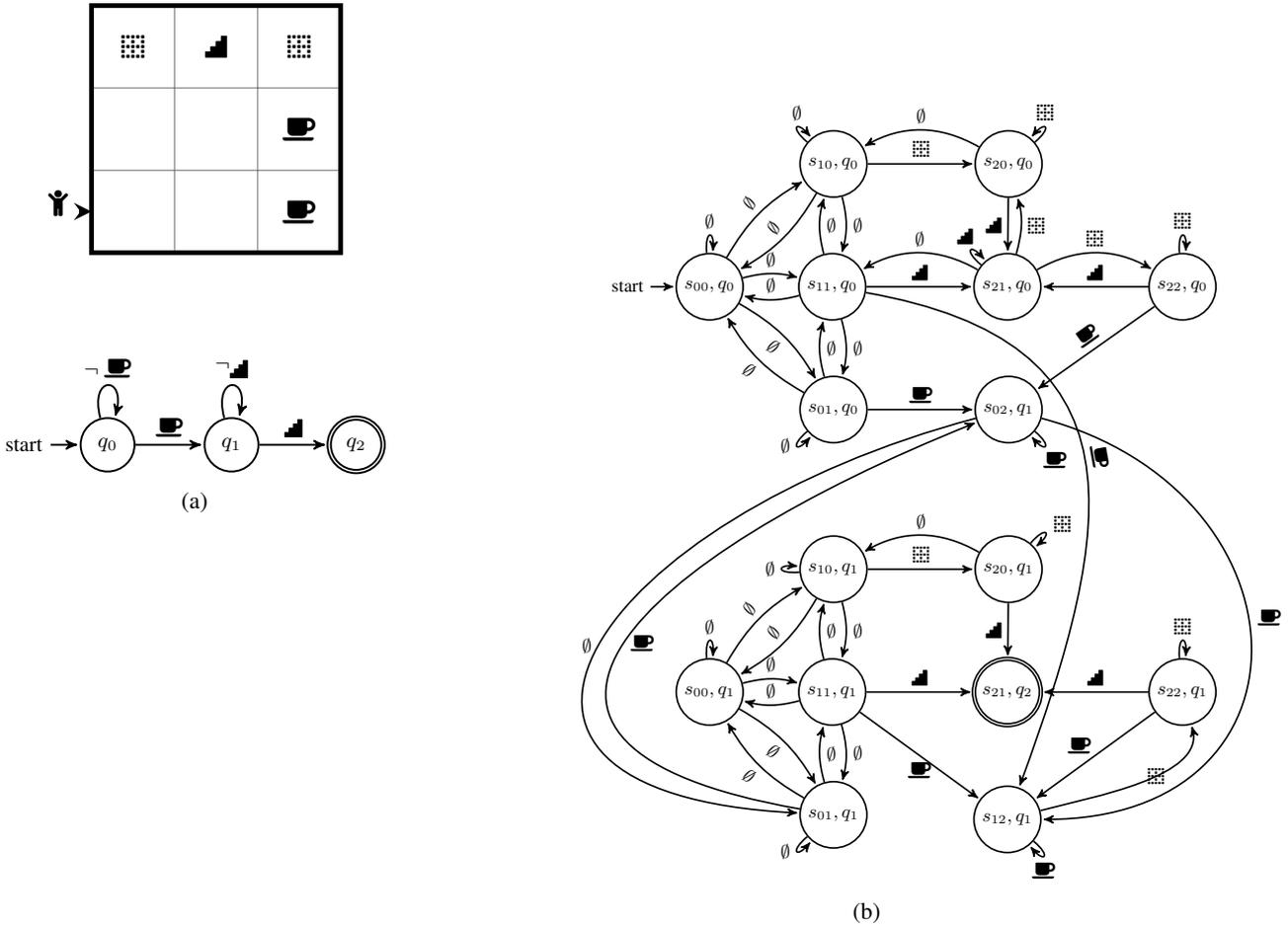
\begin{figure}[h]
  \centering
 \begin{subfigure}[b]{.3\linewidth}
 \centering
 
 \resizebox{0.8\linewidth}{!}{
 \begin{tikzpicture}[scale=1.0]
    \draw[step=1cm,gray] (0,0) grid (3, 3);
    
    \node at (2.5,0.5) {\faCoffee};
    \node at (2.5,1.5) {\faCoffee};
    
    
    \node at (0.5,2.5) {\faBorderNone};
    \node at (2.5,2.5) {\faBorderNone};
    
    \node at (-0.4,.6) {\agent};    
    \node at (1.5,2.5) {$\stairwayupfilled$};
    
    \draw[ultra thick] (0,0) rectangle (3, 3);
    
    \path[ultra thick,->, >=stealth] (-.15,.5) edge node [left] {} (0,.5);
\end{tikzpicture}
 }
 
 \vspace{4em}
  \resizebox{1\linewidth}{!}{
 \begin{tikzpicture}[->,>=stealth',shorten >=1pt,auto,node distance=1cm, thick] 

	\node[state,initial] (q_0)   {$q_0$}; 
	\node[state] (q_1) [right=of q_0] {$q_1$}; 
	\node[state,accepting] (q_2) [right=of q_1] {$q_2$}; 
	
	\path[->] 
	(q_0) edge [loop above] node {$\neg$ \faCoffee} ()   	
	(q_0) edge  node {\faCoffee} (q_1)
	
	(q_1) edge [loop above] node {$\neg \stairwayupfilled$} ()
	
	(q_1) edge node {$\stairwayupfilled$} (q_2);
\end{tikzpicture}
 }
 \subcaption{}
 \label{fig:test1mdp}
 \end{subfigure}
 \begin{subfigure}[h]{0.69\linewidth}
  \centering
  
  \resizebox{1\linewidth}{!}{
  \begin{tikzpicture}[->,>=stealth',shorten >=1pt,auto,node distance=3cm, thick]
	\tikzstyle{every state}=[fill=white,draw=black,text=black]
	
	\node[initial,state] (s00q0) {$s_{00},q_0$};
	
	\node[state]         (s10q0) [above right of=s00q0] {$s_{10},q_0$};
	\node[state]         (s11q0) [right of=s00q0, xshift=-0.9cm] {$s_{11},q_0$};
	\node[state]         (s01q0) [below right of=s00q0] {$s_{01},q_0$};
	\node[state]         (s20q0) [right of=s10q0] {$s_{20},q_0$};
	\node[state]         (s21q0) [right of=s11q0] {$s_{21},q_0$};
	\node[state]         (s02q1) [right of=s01q0] {$s_{02},q_1$};
	
	\node[state]         (s22q0) [right of=s21q0] {$s_{22},q_0$};
	
	\node[state]         (s00q1) [below of=s00q0, yshift=-4cm] {$s_{00},q_1$};
	\node[state]         (s10q1) [above right of=s00q1] {$s_{10},q_1$};
	\node[state]         (s11q1) [right of=s00q1, xshift=-0.9cm] {$s_{11},q_1$};
	\node[state]         (s01q1) [below right of=s00q1] {$s_{01},q_1$};
	
	\node[state]         (s20q1) [right of=s10q1] {$s_{20},q_1$};
	\node[state,accepting]         (s21q2) [right of=s11q1] {$s_{21},q_2$};
	
	\node[state]         (s22q1) [right of=s21q2] {$s_{22},q_1$};
	
	\node[state]         (s12q1) [below of=s21q2, yshift=0.8cm] {$s_{12},q_1$};

	\path 
	
	(s00q0) edge [bend left=15] node[anchor=south,sloped]{$\emptyset$} (s10q0)
	(s10q0) edge [bend left=15] node[anchor=south,sloped]{$\emptyset$} (s00q0)

	(s00q0) edge [bend left=15] node[anchor=north,sloped]{$\emptyset$} (s01q0)
	(s01q0) edge [bend left=15] node[anchor=north,sloped]{$\emptyset$} (s00q0)
	
	(s10q0) edge [bend left=15] node[anchor=west]{$\emptyset$} (s11q0)
	(s11q0) edge [bend left=15] node[anchor=west]{$\emptyset$} (s10q0)

	(s01q0) edge [bend left=15] node[anchor=west]{$\emptyset$} (s11q0)
	(s11q0) edge [bend left=15] node[anchor=west]{$\emptyset$} (s01q0)

	(s10q0) edge [left]      node[anchor=south]{\faBorderNone} (s20q0)
	(s20q0) edge [bend right]  node[anchor=south]{$\emptyset$} (s10q0)
	(s11q0) edge [left]      node[anchor=south]{$\stairwayupfilled$} (s21q0)
	(s21q0) edge [bend right]  node[anchor=south]{$\emptyset$} (s11q0)
	(s01q0) edge [left]      node[anchor=south]{\faCoffee} (s02q1)
	(s20q0) edge[left] node[anchor=east]{$\stairwayupfilled$} (s21q0)
	(s21q0) edge[in=-75, out=75,looseness=1] node[anchor=west]{\faBorderNone} (s20q0)
	(s21q0) edge[bend left] node[anchor=south]{\faBorderNone} (s22q0)
	(s22q0) edge[left] node[anchor=south]{$\stairwayupfilled$} (s21q0)
	(s22q0) edge[left] node[anchor=south,sloped]{\faCoffee} (s02q1)

	
	(s00q1) edge [bend left=15] node[anchor=south,sloped]{$\emptyset$} (s10q1)
	(s10q1) edge [bend left=15] node[anchor=south,sloped]{$\emptyset$} (s00q1)

	(s00q1) edge [bend left=15] node[anchor=north,sloped]{$\emptyset$} (s01q1)
	(s01q1) edge [bend left=15] node[anchor=north,sloped]{$\emptyset$} (s00q1)

	(s10q1) edge [bend left=15] node[anchor=west]{$\emptyset$} (s11q1)
	(s11q1) edge [bend left=15] node[anchor=west]{$\emptyset$} (s10q1)

	(s01q1) edge [bend left=15] node[anchor=west]{$\emptyset$} (s11q1)
	(s11q1) edge [bend left=15] node[anchor=west]{$\emptyset$} (s01q1)

	(s10q1) edge [left]      node[anchor=south]{\faBorderNone} (s20q1)
	(s20q1) edge [bend right]      node[anchor=south]{$\emptyset$} (s10q1)
	(s11q1) edge [left]      node[anchor=south]{$\stairwayupfilled$} (s21q2)
	
	(s20q1) edge[left] node[anchor=east]{$\stairwayupfilled$} (s21q2)
	(s22q1) edge[left] node[anchor=south]{$\stairwayupfilled$} (s21q2)
	(s22q1) edge[left] node[anchor=east, yshift=0.18cm, xshift=0.05cm]{\faCoffee} (s12q1)
	(s12q1) edge[in=-75, out=15,looseness=0.8] node[anchor=west, yshift=0.1cm]{\faBorderNone} (s22q1)
	
	(s01q1) edge [in=202, out=170,looseness=2.2,color=black]      node[anchor=west]{\faCoffee} (s02q1)
	
	(s02q1) edge [in=177, out=195,looseness=2.3,color=black]      node[anchor=east,color=black]{$\emptyset$} (s01q1)
	
	(s02q1) edge [in=0, out=-15,looseness=1.8,color=black]  node[anchor=west,yshift=0.5cm]{\faCoffee} (s12q1)
	
	(s11q0) edge [in=70, out=-10,looseness=1.4,color=black] node[anchor=south,sloped]{\faCoffee} (s12q1)
	
	(s11q1) edge [left]      node[anchor=north]{\faCoffee} (s12q1)
	
	
	(s00q1) edge [bend left=15] node[anchor=south]{$\emptyset$} (s11q1)
	(s11q1) edge [bend left=15] node[anchor=south]{$\emptyset$} (s00q1)

	(s00q0) edge [bend left=15] node[anchor=south]{$\emptyset$} (s11q0)
	(s11q0) edge [bend left=15] node[anchor=south]{$\emptyset$} (s00q0)

	(s00q0) edge [in=85, out=95,looseness=10] node[anchor=south]{$\emptyset$} (s00q0)
	
	(s10q0) edge [in=130, out=140,looseness=10] node[anchor=south]{$\emptyset$} (s10q0)
	
	(s01q0) edge [in=-130, out=-140,looseness=10] node[anchor=east]{$\emptyset$} (s01q0)
	
	(s20q0) edge [in=50, out=40,looseness=10] node[anchor=south]{\faBorderNone} (s20q0)
	
	(s21q0) edge [in=130, out=140,looseness=10] node[anchor=south,xshift=-0.1cm]{$\stairwayupfilled$} (s21q0)
	
	(s02q1) edge [in=-50, out=-40,looseness=10] node[anchor=west,yshift=-0.25cm,xshift=-0.15cm]{\faCoffee} (s02q1)
	
	(s22q0) edge [in=85, out=95,looseness=10] node[anchor=south]{\faBorderNone} (s22q0)

	(s00q1) edge [in=85, out=95,looseness=10] node[anchor=south]{$\emptyset$} (s00q1)
	
	(s10q1) edge [in=175, out=185,looseness=10] node[anchor=east]{$\emptyset$} (s10q1)
	
	(s01q1) edge [in=-130, out=-140,looseness=10] node[anchor=east]{$\emptyset$} (s01q1)
	
	(s20q1) edge [in=40, out=50,looseness=10] node[anchor=south,xshift=0.3cm,yshift=-0.1cm]{\faBorderNone} (s20q1)
	
	(s22q1) edge [in=85, out=95,looseness=10] node[anchor=south]{\faBorderNone} (s22q1)
	
	(s12q1) edge [in=-50, out=-40,looseness=10] node[anchor=north]{\faCoffee} (s12q1)
	
	;
\end{tikzpicture}}
  \subcaption{}
  \label{fig:test1prodmdp}
 \end{subfigure}
 \caption{A simple MDP (a) and TA (b). (c) The NFA underlying the product MDP constructed from this MDP and TA under a fully mixed policy $\pi$. MDP states $\mathcal{S}$ are labelled $s_{yx}$ (for $x, y \in \{0,1,2\}$) corresponding to the horizontal ($x$) and vertical ($y$) cell position in the grid.}
 \label{fig:test1}
\end{figure}

\newpage

\section{Baum-Welch Algorithm}
\label{sec:BW}

This section more completely introduces the Baum-Welch algorithm \citep{baum1966statistical} that we use to implement Step~1 of our TA learning pipeline. Recall there are two stages of Baum-Welch: i) learning the transition probability distribution $\hat{P}$ of the MDP $\mathcal{M}$ and ii) learning the transition probability distribution $P^\otimes$ of the product MDP $\mathcal{M}\otimes \auto$. Both follow the same procedure, where the result of the former is used for the initial parameters (an inductive bias) of the latter (see Section 3.1). Since it is the ultimate goal of this procedure, we spell out the procedure for the latter where the hidden Markov model is $\mathcal{P}^\otimes_\pi = (\mathcal{S}^\otimes, s_o^\otimes, P^\otimes_\pi, R^\otimes, \gamma, \mathcal{O}, Z)$ (constructed from the product MDP and a fully mixed policy $\pi$). Baum-Welch uses an `expectation-maximisation' method to find the hidden transition matrix and the observation (or emissions) distribution $Z$ that maximises the likelihood of obtaining the given set of observation sequences,
$$O = (o_1,..., o_T) \in (\mathcal{O}^+)^T.$$ 
Formally, we represent this optimisation problem as:
\begin{equation*}
    \langle \hat{P}^\otimes_\pi, \hat{Z}\rangle = \argmax_{\langle P^\otimes_\pi,Z\rangle} \Pr\left(O \mid \langle P^\otimes,Z \rangle  \right).
\end{equation*}

\noindent\textbf{Defining HMM Parameters.} The hidden Markov chain's state space is the set $\mathcal{S}^\otimes$ and, as is commonly done, we will view the the transition probability function $P^\otimes_\pi:\mathcal{S}^\otimes\times \mathcal{S}^\otimes\rightarrow [0,1]$ as a matrix, with elements
\begin{equation*}
    P^\otimes_{\pi,ij} \equiv P^\otimes(s^\otimes_t = s^\otimes_j\mid s^\otimes_{t-1}=s^\otimes_i).
\end{equation*}
This representation is induced by a bijective mapping, so no information is lost.

We do the same for the distribution $Z$, but call the matrix $E^Z$ (for `emissions from Z'):
\begin{equation*}
    E^Z_{ij} \equiv Z(o_t = o_i \mid s^\otimes_t = s^\otimes_j).
\end{equation*}
Here, $o_i$ represents the $i$-th element of the finite observation set $\mathcal{O}$ under any fixed ordering. Thus, $E^Z_{ij}$ is the probability that observation $\mathcal{O}_i$ is emitted when the hidden state is $s^\otimes_j$.

Under these definitions, $P^\otimes_\pi$ is an $|\mathcal{S}^\otimes|\times|\mathcal{S}^\otimes|$ matrix and $E^Z$ is an $|\mathcal{S}^\otimes|\times|\mathcal{O}|$ matrix.

Lastly, we define the initial state distribution $\rho \in [0,1]^{|\mathcal{S}^\otimes|}$ by
\begin{equation*}
    \rho_j \equiv \Pr[s^\otimes_0 = s^\otimes_j].
\end{equation*}
In our application, the starting state is always known to be the same, so this distribution will always have support 1.

\noindent\textbf{Initialisation.} We initialise $P^\otimes_\pi$ using the fact that we have an estimate of the transition matrix for the MDP $\hat{P}$ as an inductive bias and we initialise $E^Z$ according to the deterministic observation mapping $\langle s,q\rangle \mapsto \langle s, \chi_\mathcal{F}(q)\rangle$ (see Appendix D). 
\begin{align*}
    \sum_{\{s^\otimes_j : s^\otimes_j\in \mathcal{S}^\otimes\}} P^\otimes_{\pi, ij} &= 1\\
    \sum_{\{o_j: o_j\in \mathcal{O}\}} E^Z_{ij} &= 1.
\end{align*}

\noindent\textbf{Forward Procedure.} This first variable is the probability that the input model is in hidden state $s^\otimes_i$ at time $t$ in the given observation sequence $o_{t+1},\ldots,o_\mathcal{T}$:
\begin{equation*}
    \alpha_i(t)\equiv \Pr\left[o_1,\ldots,o_t,s^\otimes_t=s^\otimes_i\mid \langle P^\otimes_\pi,Z\rangle\right]
\end{equation*}
Then, the following recursive relations are used to compute these variables for later use in the update step:
\begin{align*}
    1.&~ \alpha_i(1) = \pi_i E^Z_{i,o_1} ~~~~~~~~~~~~~~~~~~~~~~~~~~~~~~~~~~~~~~~~~~~~~~~~\\
    2.&~ \alpha_i(t+1)=E^Z_{i,o_{t+1}}\sum_{j=1}^{|\mathcal{S}^\otimes|}\alpha_j(t)P^\otimes_{\pi,ji} ~~~~~~~~~~~~~~~~~~~~~~~~~~~~~~~~~~~~~~~~~~~~~~~~
\end{align*}

\noindent\textbf{Backward Procedure.} These variables are calculated backwards temporally for `smoothing' of the calculated probabilities:
\begin{equation*}
    \beta_i(t)\equiv \Pr\left[o_{t+1},\ldots,o_\mathcal{T}\mid s^\otimes_t=i, \langle P^\otimes_\pi,Z\rangle\right]
\end{equation*}
Then, the following recursive relations are used to compute these variables for use in the update step:
\begin{align*}
    1.&~ \beta_i(\mathcal{T}) = 1 ~~~~~~~~~~~~~~~~~~~~~~~~~~~~~~~~~~~~~~~~~~~~~~~~\\
    2.&~ \beta_i(t)=\sum_{j=1}^{|\mathcal{S}^\otimes|}\beta_j(t+1)P^\otimes_{\pi,ij}E^Z_{o_{t+1},j}. ~~~~~~~~~~~~~~~~~~~~~~~~~~~~~~~~~~~~~~~~~~~~~~~~
\end{align*}

\noindent\textbf{Calculate Update Parameters.} The first update parameter is
\begin{equation*}
    \gamma_i(t)= \Pr[s^\otimes_t=i\mid (o_t)_{t=1}^\mathcal{T}, \langle P^\otimes_\pi,Z\rangle] = \frac{\Pr[s^\otimes_t=s^\otimes_i,(o_t)_{t=1}^\mathcal{T}\mid \langle P^\otimes_\pi,Z\rangle]}{\Pr[(o_t)_{t=1}^\mathcal{T}\mid \langle P^\otimes_\pi,Z\rangle]} = \frac{\alpha_i(t)\beta_i(t)}{\sum_{j=1}^{|\mathcal{S}^\otimes|}\alpha_j(t)\beta_j(t)}.
\end{equation*}
Note that by the definitions of $\alpha_i(t)$ and $\beta_i(t)$, $\sum_{j=1}^{|\mathcal{S}^\otimes|}\alpha_j(t)\beta_j(t)$ is a sum over all possible hidden states which could occur at time $t$ that manifest the same observation sequence $(o_t)_{t=1}^\mathcal{T}$.

The other update parameter is
\begin{align*}
    \xi_{ij}(t) &= \Pr[s^\otimes_t=s^\otimes_i,s^\otimes_{t+1}=s^\otimes_j \mid (o_t)_{t=1}^\mathcal{T}, \langle P^\otimes_\pi,Z\rangle] \\&= \frac{\Pr[s^\otimes_t=i,s^\otimes_{t+1}=j, (o_t)_{t=1}^\mathcal{T} \mid \langle P^\otimes_\pi,Z\rangle]}{\Pr[(o_t)_{t=1}^\mathcal{T}\mid \langle P^\otimes_\pi,Z\rangle]}~~~~\textup{(by Bayes' Theorem)}\\
    &= \frac{\alpha_i(t)P^\otimes_{\pi,ij}\beta_j(t+1)E^Z_{o_{t+1},j}}{\sum_{u,v\in \mathcal{S}^\otimes}\alpha_u(t)P^\otimes_{\pi,uv}\beta_v(t+1)E^Z_{o_{t+1},v}},
\end{align*}
where the final denominator here is a sum over all possible hidden state transitions from time $t$ to $t+1$ which can give rise to the specific observation sequence $(o_t)_{t=1}^\mathcal{T}$.

\noindent\textbf{Update Step.} We calculate the actual model updates with the following equations:
\begin{align*}
    \Tilde{P}^\otimes_{ij} &= \frac{\sum_{t=0}^{\mathcal{T}-1} \xi_{ij}(t)}{\sum_{t=0}^{\mathcal{T}-1} \gamma_{i}(t)}
        = \frac{\sum_{t=0}^{\mathcal{T}-1} \Pr[s^\otimes_t=i,s^\otimes_{t+1}=j\mid (o_t)_{t=1}^\mathcal{T}, \langle P^\otimes_\pi,Z\rangle]}{\sum_{t=0}^{\mathcal{T}-1} \Pr[s^\otimes_t=i\mid (o_t)_{t=1}^\mathcal{T}, \langle P^\otimes_\pi,Z\rangle]}\\
        \Tilde{E}^Z_{ji} &=\frac{\sum_{t=0}^\mathcal{T} \chi_{o_t=j}(o_t)\gamma_i(t)}{\sum_{t=0}^\mathcal{T}\gamma_i(t)},
\end{align*}
    
where $\chi_{o_t=j}$ is the indicator function for $o_t=j\in\mathcal{O}$.
    
After this step, we re-define $P^\otimes_\pi:=\Tilde{P}^\otimes_\pi$ and $E^Z := \Tilde{E}^Z$, and repeat all the steps with the same observation data $(o_t)_{t=1}^\mathcal{T}$. Baum-Welch implementations usually keep track of the magnitude of changes from iteration to iteration, terminating when changes become negligible, or when a user-defined maximum number of iterations is reached.

\section{Mathematical Implementation of HMM Learning}\label{sec:imp}
Here we discuss any non-trivial linear algebra used by us to encode the models and data structures for the Baum-Welch algorithm.

As usual for MDPs, given the choice of some action $a\in\mathcal{A}$ at any time $t$, the transition matrix $\mathcal{T}_{P^\otimes_\pi}^a \in [0,1]^{|\mathcal{S}| |\mathcal{Q}|\times |\mathcal{S}| |\mathcal{Q}|}$ corresponding to $P^\otimes_\pi$ in $\mathcal{M}\otimes \auto$ is
\begin{equation}\label{eq:transtensor}
    \left(\mathcal{T}_{P^\otimes_\pi}^a\right)_{\langle s,q\rangle,\langle s',q'\rangle} \equiv P^\otimes_\pi\left(\langle s',q'\rangle\mid \langle s, q\rangle, a\right).
\end{equation}

The primary non-trivial linear algebra tool used is the \textit{Kronecker product}. Usually, this is denoted with the symbol $\otimes$, but we use $\otimes_K$ to avoid confusion with our notation for the product MDP.
\begin{definition}
The Kronecker product of two matrices $A, B$ is the block-matrix
\begin{equation*}
    A\otimes_K B \equiv 
        \begin{pmatrix}
        A_{11}B & A_{12}B & \ldots & A_{1m}\\
        A_{21}B & A_{22}B & & \\
        \vdots & & \ddots  & \\
        A_{n1}B& & & A_{nm}B
        \end{pmatrix},
\end{equation*}
where each $A_{ij}B$ is the $B$-sized matrix of the scalar $A_{ij}$ multiplied by $B$.
\end{definition}

The mathematical setup for converting a product MDP into an HMM is as follows (applies to any matrix-based implementation, this is not unique to our programmatic implementation).

We construct the product MDP according to Definition 4 and store it as a rank-3 tensor with each layer corresponding to the dynamics induced by a specific action. This implies that the dimensions of the tensor ($\mathcal{T}$ above) correspond to $|\mathcal{S}\times\mathcal{Q}| \times |\mathcal{S}\times\mathcal{Q}| \times |\mathcal{A}|$. Each `action layer' is equal to the matrix given by Equation \ref{eq:transtensor} above.

 We convert this product MDP to a Markov chain by summing and normalising to the fully mixed policy $\pi$. This results in an $|\mathcal{S}\times\mathcal{Q}| \times |\mathcal{S}\times\mathcal{Q}|$ transition matrix, representing the \textit{true} transition dynamics of the Hidden Markov model we will be looking to solve.

The HMM learning algorithm takes as input a set of observation sequences and initial guesses for the transition and emission dynamics of the system. Note that this requires assuming a fixed number of hidden states in the system. Since we know the number of MDP states, estimating the right number of hidden states corresponds to guessing the number of TA states $k$, which is unknown a priori. However, all that is actually required is that our guess $k$ is greater than or equal to the true number of TA states $|\mathcal{Q}|$ because guessing too many states will just generate duplication in the learnt matrix - which can easily be identified. It is easy to come up with an upper bound on the size of the underlying TA that is sufficient for there to be enough hidden states in the learned product MDP to capture the dynamics over $\mathcal{S} \times \mathcal{Q}$; for instance, we could base it on the cardinality of the set of labels or on some type of simplicity assumption that the task specification is unlikely to have more than 10 states perhaps. We then initialise \texttt{P\_initial\_estimate} as a $k|\mathcal{S}| \times k|\mathcal{S}|$ stochastic matrix by taking the Kronecker product between a $k \times k$ identity matrix and $\hat{P}_\pi$ (the estimate for the MDP's transition matrix after being induced by the fully mixed policy $\pi$). 

The observations seen consist of the MDP state component $s$, label $L(s)$, and reward. This corresponds to setting the true emissions matrix (also known as the observation probability matrix) $E^Z$ to the block matrix
\begin{equation*}
    E^Z_k :=
        \begin{blockarray}{ccc}
        & \mathcal{S}\times\{r=0\} & \mathcal{S}\times\{r=1\}\\
        \begin{block}{c(cc)}
            \mathcal{S}\times \{q_0\} & I & \textbf{0} \\
            \mathcal{S}\times \{q_1\} & I & \textbf{0} \\
            \vdots & \vdots & \vdots\\
            \mathcal{S}\times \{q_{k-1}\}& I & \textbf{0} \\
            \mathcal{S}\times \{q_k\} & \textbf{0} & I\\
        \end{block}
        \end{blockarray} = 
        \begin{pmatrix}
        1&0\\
        1&0\\
        \vdots&\vdots\\
        1&0\\
        0&1
        \end{pmatrix}
        \otimes_K
        I_{|\mathcal{S}|\times|\mathcal{S}|},
\end{equation*}
    where $I_{|\mathcal{S}|\times|\mathcal{S}|}$ is the $|\mathcal{S}|\times|\mathcal{S}|$ identity matrix, and the adjacent binary matrix is $k\times 2$. Thus, for all product states with TA state $q\neq q_k$, the observation is just the physical state and reward 0, i.e. $\langle s,0\rangle$, and in the case $q=q_k$, the TA is an accepting state (by the assumption that there is one accepting TA state and that we number it with $k$)\footnote{In our application, accepting states cannot be left after achieving the goal, so it suffices in every case to have just one accepting state.}, so the observation for $\langle s,q_k\rangle$ is $\langle s, 1\rangle$. Note that the mission output from each specific product state is deterministic in our setting.

\section{Further experimental details}

Table \ref{tab:exp1_params} contains the key hyperparameters used in our experiments and presents the precise mean times to convergence across 3 runs of each experiment.

\begin{table}
 \caption{Parameters and results for the experiments with results displayed in Figures 4a and 4b, where the mean time to convergence for the Baum-Welch algorithm over three runs was measured for varying grid-world and TA sizes.}
    \centering
    \scalebox{1}{
    \begin{tabular}{| c | c | c | c | c | r |}
    \hline
    Algorithm & Grid size & TA size & \multicolumn{1}{|p{1cm}|}{\centering Episode\\Length} & No. Episodes & \multicolumn{1}{|p{1.7cm}|}{\centering Convergence\\Time (s)}\\
    \hline
    Ours & $3 \times 3$ & 3 & 34 & 275 & 26.3\\
    Ours & $3 \times 3$ & 4 & 34 & 275 & 382.3\\
    Ours & $3 \times 3$ & 5 & 70 & 500 & 5,232.7\\
    BF & $3 \times 3$ & 3 & 34 & 275 & 63.7\\
    BF & $3 \times 3$ & 4 & 34 & 275 & 18,188.3\\
    BF & $3 \times 3$ & 5 & 70 & 500 & Timeout (>150,000)\\
    Ours & $4 \times 4$ & 3 & 80 & 500 & 660.7\\
    Ours & $4 \times 4$ & 4 & 90 & 1000 & 3,824.7\\
    Ours & $4 \times 4$ & 5 & 80 & 1000 & 38,025.0\\
    Ours & $5 \times 5$ & 3 & 85 & 2000 & 2,625.7\\
    Ours & $5 \times 5$ & 4 & 100 & 2000 & 16,819.7\\
    Ours & $5 \times 5$ & 5 & 140 & 2000 & 68,977.7\\
    \hline
    \end{tabular}}
    \label{tab:exp1_params}
\end{table}

\section{Comparisons with Related work}

We attempted to benchmark our algorithm against several existing passive DFA learning libraries. Two well-established and commonly-used approaches to this are the RPNI algorithm (used in, e.g., \cite{xu2020joint,xu2021active}) and SAT-based algorithms (used in e.g., \citep{deepsynth,xu2020joint,verginis2022joint, abadi2020learning, corazza2022reinforcement}). The results of our comparisons are summarised in Table \ref{tab:DFA_libraries}. Of the six implementations we tried, only the SAT-based approach from the libalf library was able to successfully learn the 3-state TA in our 3x3 gridworld. As Table \ref{tab:DFA_libraries} and Figure 4b demonstrate, our algorithm runs significantly faster than the RPNI algorithm indicating its better efficiency compared to existing approaches. Additionally, most of the libraries we attempted to run using the same dataset that we used for our experiments were not able to learn the correct TA, as evidenced by the results in \ref{tab:DFA_libraries}. This highlights the general applicability of our approach in passively learning DFAs.

\begin{table}
\caption{Results from attempted comparisons to existing passive DFA learning libraries.}
    \centering
    \begin{tabular}{| c | c | c | l |}
        \hline
        Library & Algorithm & Result & URL \\
        \hline
        AALpy & RPNI & Wrong DFA (Fig. 3) & \url{https://github.com/DES-Lab/AALpy} \\
        dfa-identify & SAT-based & Timeout & \url{https://github.com/mvcisback/dfa-identify} \\
        Inferrer & RPNI & Recursion error & \url{https://github.com/steynvl/inferrer} \\
        FlexFringe & SAT-based & Throws error & \url{https://github.com/tudelft-cda-lab/FlexFringe} \\
        Libalf & RPNI & Wrong DFA & \url{https://github.com/libalf/libalf}\\
        Libalf & SAT-based & Works & \url{https://github.com/libalf/libalf}\\
        \hline
    \end{tabular}
    \label{tab:DFA_libraries}
\end{table}

Learning an automaton from data is an NP-complete problem \citep{gold1978complexity}, so under standard complexity assumptions, the convergence time will, in the worst case, always grow exponentially with instance size. Nevertheless, as remarked in our experimental section, a 5x5 MDP grid world with a 5-state TA is comparable to that of related research. For example, in \citep{abadi2020learning} their largest problem is the ``Maze domain'' with a five-state TA, MDP size 4x4 and in \citep{gaonNMR} their largest is a five-state TA and a 5x5 MDP. Other works such as \citep{furelos2020induction,xu2020joint,lauffer2022learning} do go to slightly larger MDP worlds, but they employ SAT-based approaches, which assume knowledge of the underlying environment, unlike ours.

\begin{figure}
    \includepdf[pages=-,scale=0.8]{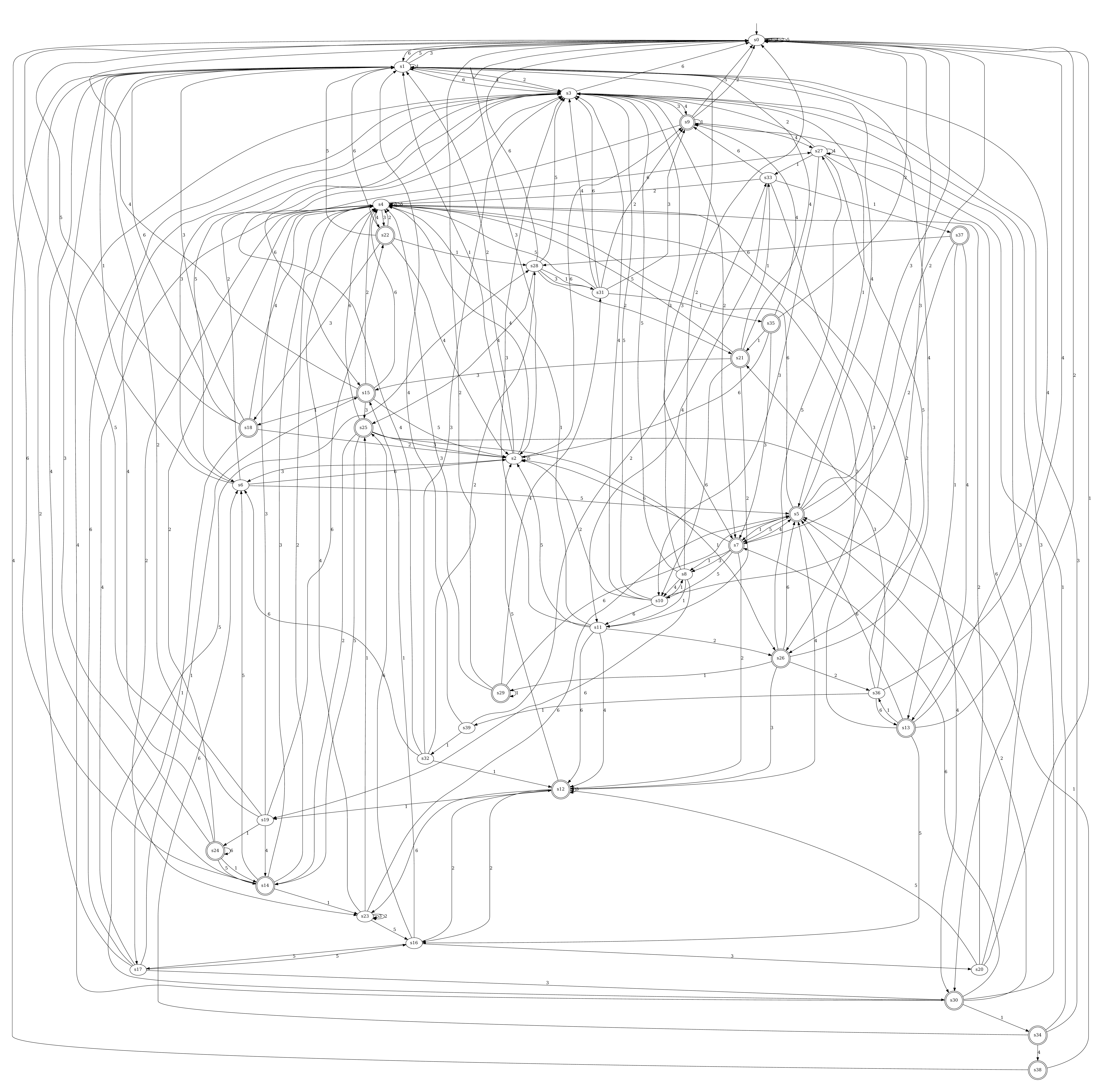}
    \vspace{60em}  
    \caption{TA learnt using the AALpy library with the same data set used to learn the 3-state TA in the 3x3 gridworld.}
    \label{fig:AALpy_TA}
\end{figure}

\end{document}